\documentclass{article}
\usepackage[utf8]{inputenc}
\usepackage{geometry}
\newgeometry{vmargin={1in}, hmargin={1.25in,1.25in}}
\usepackage[T1]{fontenc}    %
\usepackage{hyperref}       %
\usepackage{url}            %
\usepackage{booktabs}       %
\usepackage{amsfonts}       %
\usepackage{nicefrac}       %
\usepackage{microtype}      %
\usepackage{xcolor}         %

\usepackage{amsmath}
\usepackage{amssymb}
\usepackage{mathtools}
\usepackage{amsthm}

\usepackage{cleveref}
\crefname{algorithm}{Algorithm}{Algorithms}
\crefname{assumption}{Assumption}{Assumptions}
\crefname{equation}{}{}
\crefname{figure}{Fig.}{Figs.}
\crefname{table}{Table}{Tables}
\crefname{section}{Section}{Sections}
\crefname{subsection}{Section}{Sections}
\crefname{theorem}{Theorem}{Theorems}
\crefname{lemma}{Lemma}{Lemmmas}
\crefname{proposition}{Proposition}{Propositions}
\crefname{definition}{Definition}{Definitions}
\crefname{corollary}{Corollary}{Corollaries}
\crefname{remark}{Remark}{Remarks}
\crefname{example}{Example}{Examples}
\crefname{appendix}{Appendix}{Appendices}

\usepackage{amsmath}
\usepackage{amsthm}
\allowdisplaybreaks
\usepackage{amssymb}
\usepackage{algorithm}
\usepackage[lined,boxed,ruled,norelsize,algo2e,linesnumbered]{algorithm2e}
\usepackage{bbm}

\definecolor{b2}{RGB}{51,153,255}
\definecolor{mygreen}{RGB}{80,180,0}

\theoremstyle{plain}
\newtheorem{theorem}{Theorem}[section]
\newtheorem{lemma}[theorem]{Lemma}

\newtheorem{assumption}[theorem]{Assumption}

\theoremstyle{definition}
\newtheorem{definition}[theorem]{Definition}

\theoremstyle{remark}
\newtheorem{remark}[theorem]{Remark}

\newcommand{\wh}{\widehat}

\newcommand{\eps}{\varepsilon}
\renewcommand{\epsilon}{\varepsilon}
\renewcommand{\phi}{\varphi}

\newcommand{\R}{\mathbb{R}}

\newcommand{\F}{\mathcal{F}}

\newcommand{\calZ}{\mathcal{Z}}
\newcommand{\calD}{\mathcal{D}}

\newcommand{\calN}{\mathcal{N}}
\newcommand{\calS}{\mathcal{S}}

\newcommand{\calO}{\mathcal{O}}

\renewcommand{\hat}{\wh}

\DeclareMathAlphabet{\mathpzc}{OT1}{pzc}{m}{it}
\newcommand{\nSG}{\mathrm{nSG}}

\newcommand{\ind}{\mathbf{1}}

\newcommand{\Lap}{\mathrm{Lap}}

\newcommand{\expec}{\mathbb{E}}
\newcommand{\E}{\mathbb{E}}

\newcommand{\hf}{\widehat{F}}

\usepackage{cleveref}
\crefname{algorithm}{Algorithm}{Algorithms}
\crefname{assumption}{Assumption}{Assumptions}
\crefname{equation}{}{}
\crefname{figure}{Fig.}{Figs.}
\crefname{table}{Table}{Tables}
\crefname{section}{Section}{Sections}
\crefname{subsection}{Section}{Sections}
\crefname{theorem}{Theorem}{Theorems}
\crefname{lemma}{Lemma}{Lemmmas}
\crefname{proposition}{Proposition}{Propositions}
\crefname{definition}{Definition}{Definitions}
\crefname{corollary}{Corollary}{Corollaries}
\crefname{remark}{Remark}{Remarks}
\crefname{example}{Example}{Examples}
\crefname{appendix}{Appendix}{Appendices}

\newcommand{\ZZ}{\mathcal{Z}}
\newcommand{\XX}{\mathcal{X}}
\newcommand{\DD}{\mathcal{D}}
\newcommand{\Alg}{\mathcal{A}}

\newcommand{\pr}{\mathbb{P}}
\newcommand{\AboTh}{\text{AboveThreshold}}

\title{
Faster Algorithms for User-Level Private Stochastic Convex Optimization\footnote{Authors are listed in reverse alphabetical order.}
}

\author{Andrew Lowy\thanks{University of Wisconsin-Madison. \texttt{alowy@wisc.edu}} \and Daogao Liu\thanks{University of Washington. \texttt{liudaogao@gmail.com}
} \and Hilal Asi\thanks{Apple Machine Learning Research. \texttt{hilal.asi94@gmail.com}
}
}

\begin{document}

\maketitle

\begin{abstract}
We study private stochastic convex optimization (SCO) under user-level differential privacy (DP) constraints. In this setting, there are $n$ users (e.g., cell phones), each possessing $m$ data items (e.g., text messages), and we need to protect the privacy of each user's entire collection of data items. Existing algorithms for user-level DP SCO are impractical in many large-scale machine learning scenarios because: (i) they make restrictive assumptions on the smoothness parameter of the loss function and require the number of users to grow polynomially with the dimension of the parameter space; or (ii) they are prohibitively slow, requiring at least $(mn)^{3/2}$ gradient computations for smooth losses and $(mn)^3$ computations for non-smooth losses. To address these limitations, we provide novel user-level DP algorithms with state-of-the-art excess risk and runtime guarantees, without stringent assumptions. First, we develop a \textit{linear-time} algorithm with state-of-the-art excess risk (for a non-trivial linear-time algorithm) under a mild smoothness assumption. Our second algorithm applies to arbitrary smooth losses and achieves \textit{optimal excess risk} in $\approx (mn)^{9/8}$ gradient computations. Third, for \textit{non-smooth} loss functions, we obtain \textit{optimal excess risk} in $n^{11/8} m^{5/4}$ gradient computations.
Moreover, our algorithms do not require the number of users to grow polynomially with the dimension. 
\end{abstract}

\section{Introduction}

The increasing ubiquity of machine learning (ML) systems in industry and society has sparked serious concerns about the privacy of the personal data used to train these systems. Much work has shown that ML models may violate individuals' privacy by leaking their sensitive training data~\cite{shokri2017membership,li2024analyzing,lowy2024does}. For instance, large language models (LLMs) are vulnerable to black-box attacks that extract individual training examples~\cite{carlini2021extracting}. \textit{Differential privacy}~(DP)~\cite{dwork2006calibrating} prevents ML models from leaking their training data.

\vspace{.2cm}
The classical definition of differential privacy---\textit{item-level differential privacy}~\cite{dwork2006calibrating}---is ill-suited 
for many modern applications. Item-level DP ensures that the inclusion or exclusion of any \textit{one training example} has a negligible impact on the model's outputs. \textit{If each person (a.k.a. \textit{user}) contributes only one piece of training data}, then item-level DP provides a strong guarantee that each user's data cannot be leaked. However, in many modern ML applications, such as training LLMs on users' data in federated learning, each user contributes a large number of training examples~\cite{Xu2024}. In such scenarios, the privacy protection that item-level DP provides for each user is insufficiently weak. 

\vspace{.2cm}
\textit{User-level differential privacy} is a stronger privacy notion that addresses the above shortcoming of item-level DP. Informally, user-level DP ensures that the inclusion or exclusion of any \textit{one user's entire training data} ($m$ samples) has a negligible impact on the model's outputs. Thus, user-level DP provides a strong guarantee that no user's data can be leaked, even when users contribute many training examples. 

\vspace{.2cm}
A fundamental problem in (private) machine learning is \textit{stochastic convex optimization} (SCO): given a data set $\DD = (Z_1, \ldots, Z_n)$ from $n$ i.i.d. users, each possessing $m$ i.i.d. samples from an unknown distribution $Z_i \sim P^m$ our goal is to approximately minimize the expected population loss \[
F(x) := \expec_{z \sim P}[f(x, z)].
\] 
Here, $f: \XX \times \ZZ \to \R$ is a loss function (e.g., cross-entropy loss), $\XX \subset \R^d$ is the parameter domain, and $\ZZ$ is the data universe. We require that the output of the optimization algorithm $\Alg: \ZZ^{mn} \to \XX$ satisfies user-level DP (Definition~\ref{def: DP}). We measure the accuracy of $\Alg$ by its \textit{excess (population) risk} 
\[
\expec F(\Alg(\DD)) - F^* := \expec_{\Alg, \DD \sim P^{nm}} F(\Alg(\DD)) - \min_{x \in \XX} F(x).
\]

Given the practical importance of user-level DP SCO, it is unsurprising that many prior works have studied this problem. The work of~\cite{levy2021learning} initiated this line of work, and provided an excess risk lower bound of $\Omega(1/\sqrt{nm} + \sqrt{d}/(\eps n \sqrt{m}))$, where $\eps$ is the privacy parameter. However, their upper bound was suboptimal and required strong assumptions. The work of~\cite{bassily2023user} gave an algorithm that achieves optimal risk for $\beta$-smooth losses with 
$\beta < (n/\sqrt{md} \wedge n^{3/2}/(d\sqrt{m}))$,
provided that $n \geq \sqrt{d}/\eps$ and $m \leq \max(\sqrt{d}, n \eps^2/\sqrt{d})$. \textit{These assumptions are restrictive} in large-scale applications with a large number of examples per user $m$ or when the number of model parameters $d$ is large. For example, in deep learning, we often have $d \gg n$ and an enormous smoothness parameter $\beta \gg 1$. 
Moreover, their algorithm requires 
$mn^{3/2}$
gradient evaluations, making it slow when the number of users $n$ is large.\footnote{In the introduction, whenever $\eps$ does not appear, we are assuming $\eps=1$ to ease readability. For runtime bounds, we also assume $n=d$ to further simplify.} The work of~\cite{ghazi2023user} gave another user-level DP algorithm that only requires $n \geq \log(d)/\eps$, but unfortunately their algorithm does not run in polynomial-time. 

\vspace{.2cm}
To address the deficiencies of previous works on user-level DP SCO, the recent work~\cite{asi2023user} provided an algorithm that achieves optimal excess risk in polynomial-time, while also only requiring $n \geq \log(md)/\eps$ users. Moreover, their algorithm also works for non-smooth losses. The drawback of~\cite{asi2023user} is that it is even \textit{slower} than the algorithm of~\cite{bassily2023user}: for $\beta$-smooth losses, their algorithm requires $\beta \cdot (nm)^{3/2}$ gradient evaluations; for non-smooth losses, their algorithm requires $(nm)^3$ evaluations. 

\vspace{.2cm}
Evidently, \textit{the runtime requirements and parameter restrictions of existing algorithms for user-level DP SCO are prohibitive in many important ML applications}. Thus, an important question is: 

\begin{center}
\noindent\fbox{
    \parbox{0.5\linewidth}{
    \vspace{-0.cm}
\textbf{Question 1.} Can we develop \textit{faster} user-level DP algorithms that achieve \textit{optimal} excess risk \textit{without restrictive assumptions}?  
}}
    \end{center}

\paragraph{Contribution 1.} We give a positive answer to Question 1, providing a novel algorithm that achieves optimal excess risk using 
$\max\{\beta^{1/4}(nm)^{9/8}, \beta^{1/2} n^{1/4}m^{5/4}\}$
gradient computations for $\beta$-smooth loss functions, with any $\beta < \infty$ (\cref{thm: utility of accelerated phased ERM}). For non-smooth loss functions, our algorithm achieves optimal excess risk using $n^{11/8} m^{5/4}$ gradient evaluations for non-smooth loss functions (\cref{thm: nonsmooth}). 
\textit{Our runtime bounds dominate those of all prior works} in every applicable parameter regime, by polynomial factors in $n, m$, and $d$. 
Moreover, our results only require $n^{1 - o(1)} \ge \log(d)/\eps$ users. %
See Table~\ref{table: summary} for a comparison of our results vs. prior works. For example, \textit{for non-smooth loss functions, our optimal algorithm is faster than the previous state-of-the-art~\cite{asi2023user} by a multiplicative factor of $n^{13/8} m^{7/4}$.} For smooth loss functions, our optimal algorithm is faster than~\cite{asi2023user} by a factor of $(nm)^{3/8} \beta^{3/4}$ (in the typical parameter regime when $n^7 \ge m$).

\begin{figure*}
\centering
        \includegraphics[width = 0.9\textwidth]{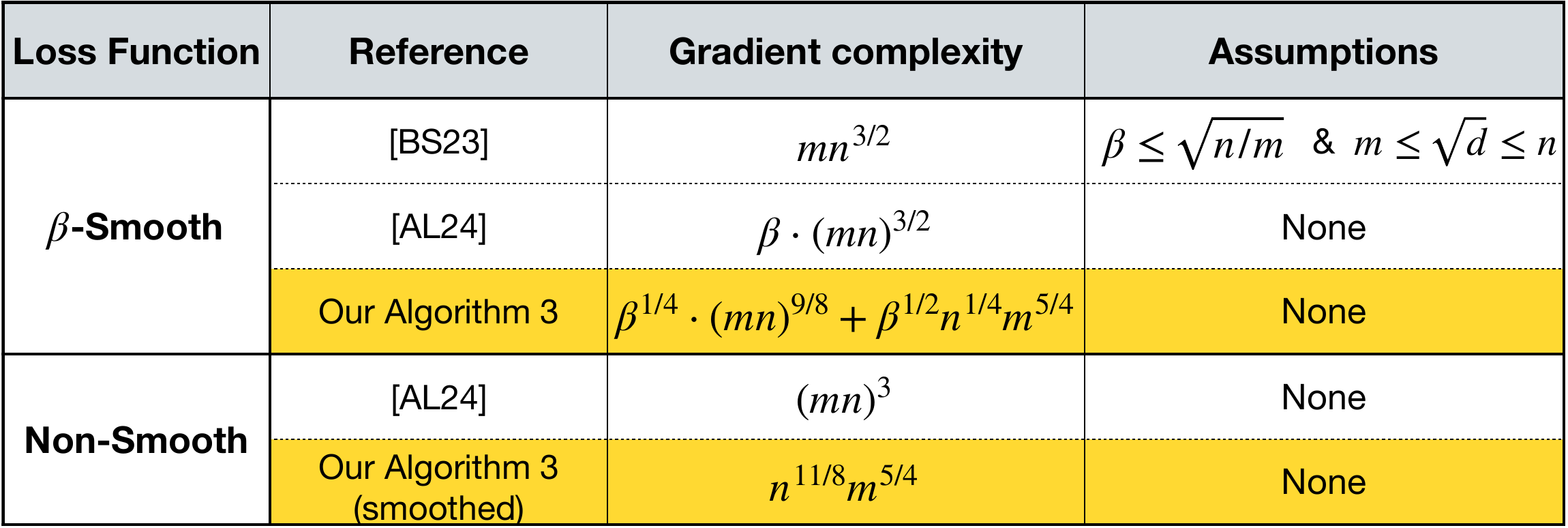}
      \vspace{-0.1in}
     \caption{\footnotesize 
Optimal algorithms for user-level DP SCO. We omit logarithms, fix $L = R = 1 = \eps$ and $n = d$. 
}\label{table: summary}
\vspace{-.1in}
\end{figure*}

\paragraph{Linear-Time Algorithms}
The ``holy grail'' of DP SCO is a \textit{linear-time} algorithm with optimal excess risk, which is unimprovable both in terms of runtime and accuracy. In the \textit{item-level} DP setting, such algorithms are known to exist for smooth loss functions~\cite{fkt20,zhang2022bring}. \cite{asi2023user} posed an interesting open question: is there a \textit{user-level} DP algorithm that achieves optimal excess risk in linear time for smooth functions? For our second contribution, we make progress towards answering this question. 

\vspace{.2cm}
Existing techniques for user-level DP SCO are not well-suited for linear-time algorithms. Indeed, the only prior non-trivial linear-time algorithm
is the user-level LDP algorithm of~\cite[Algorithm 5]{bassily2023user}.\footnote{It is trivial to achieve excess risk $\approx 1/\sqrt{nm} + \sqrt{d}/(\eps n)$ with $(\eps, \delta)$-user-level, e.g. by applying \textit{group privacy} to an optimal item-level DP algorithm such as~\cite{fkt20}. The error due to privacy in this bound does not decrease with $m$.} Their algorithm can achieve excess risk $\approx 1/\sqrt{nm\eps} + \sqrt{d}/(\sqrt{nm} \eps)$.
Unfortunately, however, their algorithm 
requires a very stringent assumption on the smoothness parameter $\beta < \sqrt{n^3/(md^3)}$, which is unlikely to hold for large-scale ML problems. Further, the result of~\cite{bassily2023user} requires the number of users queried in each round to grow polynomially with the dimension $d$, and it assumes $m < d < n$. \textit{These assumptions severely limit the applicability of~\cite[Algorithm 5]{bassily2023user} in practical ML scenarios}. This leads us to: 

\begin{center}
\noindent\fbox{
    \parbox{0.5\linewidth}{
    \vspace{-0.cm}
\textbf{Question 2.} Can we develop a \textit{linear-time} user-level DP algorithm 
with state-of-the-art excess risk,
\textit{without restrictive assumptions}?   %
}}
    \end{center}

\paragraph{Contribution 2.} We answer Question 2 affirmatively in~\cref{thm: phased erm without reg}: under a very mild requirement on the smoothness parameter $\beta < \sqrt{nmd}$, our novel linear-time algorithm achieves excess risk of $\approx 1/\sqrt{nm\eps} + \sqrt{d}/(\sqrt{nm} \eps)$. Moreover, our algorithm does not require the number of users to grow polynomially in the dimension $d$, and our result holds for any values of $m, d,$ and $n$. Thus, our algorithm has excess risk matching that of~\cite{bassily2023user}, but is much more widely applicable. 

\subsection{Techniques}
We develop novel techniques and algorithms to achieve new state-of-the-art results in user-level DP SCO. Before discussing our techniques, let us review the key ideas from prior works that we build on. 

\vspace{.2cm}
The goal of prior works~\cite{bassily2023user, asi2023user} was to develop user-level analogs of DP-SGD~\cite{bft19}, which is optimal in the item-level setting. To do so, they
observed that each user $i$'s gradient $\frac{1}{m}\sum_{j=1}^m \nabla f(x, Z_{i,j})$ lies in a ball of radius $\approx 1/\sqrt{m}$ around the population gradient $\nabla F(x)$ with high probability, if the data is i.i.d ($Z_i \sim P^m$). Consequently, if the data is i.i.d., then replacing one user $Z_i \in \DD$ by another user $Z'_i \in \DD'$ will not change the empirical gradient $\nabla F_{\DD}(x)$ by too much: $\| \nabla F_{\DD}(x) - \nabla F_{\DD'}(x)\| \lesssim 1/(n\sqrt{m})$ with high probability. Thus, one would hope for a method to privatize $\nabla F_{\DD}(x)$ by adding noise that scales with $1/(n\sqrt{m})$---rather than $1/n$---which would allow for optimal excess risk. 
\cite{asi2023user} devised such a method, which was inspired by FriendlyCore~\cite{tsfadia2022friendlycore}. Their method privately detects and removes ``outlier'' user gradients, and then adds noise to the average of the ``inlier'' user gradients. This outlier-removal procedure ensures privacy with noise scaling with $1/(n\sqrt{m})$, provided $n \gtrsim 1/\eps$. Moreover, when the data is i.i.d., no outliers will be removed with high probability, leading to a nearly unbiased estimator of the empirical gradient. 

\vspace{.2cm}
Our algorithms apply variations of the outlier-removal idea of~\cite{asi2023user} in novel ways. 

\vspace{.2cm}
Our linear-time~Algorithm~\ref{alg: phased ERM without regularization} takes a different approach to outlier removal, compared to prior works. Instead of removing outlier \textit{gradients}, we aim to detect and remove outlier SGD \textit{iterates}.\footnote{The reason that this innovation is necessary is discussed in the last paragraph of Section~\ref{sec: linear in m}.} The high-level idea of our algorithm is to partition the $n$ users into $C \approx 1/\eps$ groups, with each group containing $\approx n \eps$ users. For each group of users, we run $T \approx m n \eps$ steps of online SGD using the samples in this group and obtain the average iterate of each group: $\{\tilde{x}_j\}_{j=1}^C$. We then \textit{privately identify and remove the outlier iterates} from $\{\tilde{x}_j\}_{j=1}^C$. In order to successfully do so, we need to argue that if we run online SGD independently on user $Z$ and user $Z'$ to obtain $\tilde{x}$ and $\tilde{x}'$ respectively, then $\|\tilde{x} - \tilde{x}'\| \lesssim \eta \sqrt{T}$ with high probability, where $\eta$ is the SGD step size. We prove such a stability bound in Lemma~\ref{lemma 2}, which we hope will be of independent interest. By repeating the above process $\log(n)$ times and using iterative \textit{localization}~\cite{fkt20}, we obtain our state-of-the-art linear-time result.

\vspace{.2cm}
Our second algorithm, Algorithm~\ref{alg: accelerated phased ERM}, builds on~\cite{asi2023user} in a different way. In~Algorithm~\ref{alg: accelerated phased ERM}, we apply an outlier-removal procedure to users' gradients. However, unlike \cite{asi2023user}, we draw random \textit{minibatches} of users in each iteration and apply outlier-removal to these minibatches. To make this procedure private while also achieving optimal excess risk, we 
combine \textit{AboveThreshold}~\cite{dwork2014} with
\textit{privacy amplification by subsampling}~\cite{balle18}. We then develop an \textit{accelerated}~\cite{ghadimilan1} user-level DP algorithm that solves a carefully chosen sequence of regularized ERM problems, and applies localization in the spirit of~\cite{kll21, AsiFeKoTa21}. An obstacle that arises when we try to extend the ERM-based localization framework to the user-level DP setting is getting a tight bound on the variance of our minibatch stochastic gradient estimator that scales with $1/m$. We overcome this obstacle in Lemma~\ref{lem: variance bound}, by appealing to the \textit{stability of user-level DP}~\cite{bassily2023user}. To handle non-smooth loss functions, we apply randomized smoothing to our accelerated algorithm.

\subsection{Preliminaries}
We consider loss functions $f: \XX \times \ZZ \to \mathbb{R}$, where $\XX$ is a convex parameter domain and
$\ZZ$ is a data universe. Let $P$ be an unknown data distribution and $F(x) := \expec_{z \sim P}[f(x,z)]$ be the population loss function. Denote $F^* := \min_{x \in \XX} F(x)$. The SCO problem is $\min_{x \in \XX} F(x)$. 
Let $\| \cdot \|$ denote the $\ell_2$ norm. 
$\Pi_{\XX}(u) := \text{argmin}_{x \in \XX}\|u - x\|^2$ denotes projection onto $\XX$. 
\paragraph{Assumptions and Notation.}
Function $g: \XX \to \mathbb{R}$ is \textit{$L$-Lipschitz} if $|g(x) - g(x')| \leq L\|x - x'\|_2$ for all $x, x' \in \XX$. 
Function $g: \XX \to \mathbb{R}$ is \textit{$\beta$-smooth} if $g$ is differentiable and has $\beta$-Lipschitz gradient: $\|\nabla g(x) - \nabla g(x')\|_2 \leq \beta\|x - x'\|_2$.
Function $g: \XX \to \mathbb{R}$ is \textit{$\mu$-strongly convex} if $g(\alpha x + (1- \alpha) x') \leq \alpha g(x) + (1 - \alpha) g(x') - \frac{\alpha (1-\alpha) \mu}{2}\|x - x'\|^2$ for all $\alpha \in [0,1]$ and all $x, x' \in \XX$. If $\mu = 0,$ we say $g$ is \textit{convex}. 

\begin{assumption}
\label{ass: convex}
\begin{enumerate}
    \item The convex set $\XX$ is compact with $\|x - x'\|\leq R$ for all $x,x' \in \XX$. 
    \item The loss function $f(\cdot,z)$ is $L$-Lipschitz and convex for all $z \in \ZZ$. 
\end{enumerate}
\end{assumption}

In all of the paper \textit{except for \cref{sec: nonsmooth}}, we will also assume: 
\begin{assumption}
\label{ass: smooth}
The loss function $f(\cdot,z)$ is $\beta$-smooth for all $z \in \ZZ$. 
\end{assumption}

Denote $a \wedge b := \min(a,b)$. For functions $f$ and $g$ of input parameters $\theta$, we write $f \lesssim g$ if there is an absolute constant $C > 0$ such that $f(\theta) \leq C g(\theta)$ for all permissible values of $\theta$. We use $\widetilde{O}$ to hide logarithmic factors. Write $a \leq \text{poly}(b)$ if there exists some large $J > 1$ for which $a \leq b^J$. 

\paragraph{Differential Privacy.}
\begin{definition}[User-Level Differential Privacy]
\label{def: DP}
Let $\varepsilon \geq 0, ~\delta \in [0, 1).$ Randomized algorithm $\Alg: \ZZ^{nm} \to \mathcal{X}$ is \textit{$(\varepsilon, \delta)$-user-level differentially private} (DP) if for any two datasets $\DD = (Z_1, \ldots, Z_n)$ and $\DD' = (Z'_1, \ldots, Z'_n)$ that differ in one user's data (say $Z_i \neq Z_i'$ but $Z_j = Z'_j$ for $j \neq i$), we have \[
\mathbb{P}(\Alg(\DD) \in S) \leq e^\varepsilon \mathbb{P}(\Alg(\DD') \in S) + \delta,
\]
for all measurable subsets $S \subset \XX$. 
\end{definition}
Definition~\ref{def: DP} prevents any adversary from learning much more about an individual's data set than if that data had not been used for training. Appendix~\ref{app: prelim} contains the necessary background on DP. 

\subsection{Roadmap}
We begin with our state-of-the-art linear-time algorithm in~\cref{sec: linear in m}. In~\cref{sec: accelerated phased ERM}, we present our error-optimal algorithm with state-of-the-art runtime for smooth loss functions. \cref{sec: nonsmooth} extends our fast optimal algorithm to non-smooth loss functions. We conclude in~\cref{sec: conclusion} with a discussion and guidance on future research directions stemming from our work.

\section{A state-of-the-art linear-time algorithm for user-level DP SCO}
\label{sec: linear in m}

In this section, we develop a new algorithm (Algorithm~\ref{alg: phased ERM without regularization}) for %
user-level DP SCO that runs in linear time and has state-of-the-art excess risk, without requiring any impractical assumptions. The algorithm can be seen as a user-level DP variation of the localized phased SGD of~\cite{fkt20}: we execute a sequence of SGD trajectories with geometrically decaying step sizes, shrinking both the expected distance to the population minimizer and the privacy noise over a logarithmic number of phases. 

\vspace{.2cm}
In each phase $i$, we first re-set algorithmic parameters and draw a disjoint set of $n_i$ users $D_i \subset \DD$ (lines 4-5). We further partition $D_i$ into $C$ disjoint subsets $\{D_{i,j}\}_{j=1}^C$. For each $j \in [C]$, we pool together all of the $n_i m$ samples in $D_{i,j}$ and run one-pass online SGD on $D_{i,j}$ with initial point $x_{i-1}$ given to us from the previous phase. Next, in lines 10-20, we privately detect and remove ``outliers'' from $\{\tilde{x}_{i,j}\}_{j=1}^C$. That is, our goal is to privately select a subset $\calS_i \subset \{\tilde{x}_{i,j}\}_{j=1}^C$, such that for any two points $\tilde{x}_{i,j}, \tilde{x}_{i,j'} \in \calS_i$, $\|\tilde{x}_{i,j} - \tilde{x}_{i,j'}\| \leq \tau_i = \widetilde{O}(\eta_i L \sqrt{T_i})$. This will enable us to add noise scaling with $\tau_i$ in line 22, rather than with the much larger worst-case sensitivity (that scales linearly with $T_i$). In order to privately select such a subset $\calS_i$, we first compute (and privatize) the \textit{concentration score} for $\{\tilde{x}_{i,j}\}_{j=1}^C$ in line 10. A small concentration score indicates that outlier removal is doomed to fail and we must halt the algorithm to avoid breaching the privacy constraint. A large concentration score indicates that $\{\tilde{x}_{i,j}\}_{j=1}^C$ is nearly $\tau_i$-concentrated and we may proceed with outlier removal in lines 12-15.

\begin{algorithm2e}
\caption{User-Level DP Phased SGD with Outlier Iterate Removal and Output Perturbation}
\label{alg: phased ERM without regularization}
{\bf Input:} Dataset $\calD = (Z_1,\dots,Z_n) $, privacy parameters $(\epsilon,\delta)$, 
parameters $p, q > 0$, stepsize $\eta$\;
Set $l = \lfloor \log_2(n) \rfloor $, $C := 100 \log(20nme^{\eps}/\delta)/\eps$\;
\For{$i=1,\cdots,l$}
{ 
Set $n_i= (1 - (1/2)^q) n /2^{iq}$, $\eta_i=\eta/2^{pi}$, $N_i = n_i/C$, $T_i = N_i m$, $\tau_i = 1000 \eta_i L \sqrt{T_i} \log(ndm)$\; 
Draw disjoint users $D_{i}$ of size $n_i$ from $\calD$\;
Divide $D_i$ into $C$ disjoint subsets $\{D_{i,j}\}_{j=1}^{C}$, each containing $|D_{i,j}| = N_i$ users\;
\For{$j=1,\cdots,C$}
{
$\Tilde{x}_{i,j} \gets SGD(D_{i,j},\eta_i, T_i,x_{i-1}) =$ average iterate of $T_i$ steps of one-pass projected SGD with data $D_{i,j}$, stepsize $\eta_i$, and initial point $x_{i-1}$ \; 
}

Compute the concentration score for $D_i$:
\begin{align*}
    s_{i}(\tau_i):=\frac{1}{C}\sum_{j, j'\in [C]}\mathbf{1}(\|\Tilde{x}_{i,j}-\Tilde{x}_{i,j'}\|\le\tau_i)
\end{align*}
Let $\hat{s}_{i}(\tau_i)=s_{i}(\tau_i)+\Lap(20/\eps)$\;
\If{$\hat{s}_{i}(\tau_i)\ge \frac{4C}{5} $}
{
$\calS_i = \emptyset$ \;
\For{$j=1,\cdots,C$}
{
Compute the score function of $\tilde{x}_{i,j}$: $h_{i,j} =\sum_{j' = 1}^{C}\ind(\|\Tilde{x}_{i,j} -\Tilde{x}_{i,j'}\|\le2\tau_i)$\;
Add $\Tilde{x}_{i,j}$ to $\mathcal{S}_{i}$ with probability $p_{i,j}$ for 
$
p_{i,j} =
\begin{cases}
    0 & h_{i,j} < C /2 \\
    1 & h_{i,j} \ge 2C/3\\
    \frac{h_{i,j} -C/2}{C/6} & o.w.
\end{cases}  
$
}

}
\Else
{ 
{\bf Halt; Output 0}
}
Let $\tilde{x}_i = \frac{1}{|\mathcal{S}_i|} \sum_{\tilde{x}_{i,j} \in \mathcal{S}_i} \tilde{x}_{i,j}$ \;
$x_i\leftarrow \tilde{x}_i+\zeta_i$, where $\zeta_i\sim\calN(0,\sigma_i^2I_d)$ with $\sigma_i = \frac{100 \tau_i \log^2(n/\delta)}{\epsilon C}$\label{ln:add_gaussian}\;
}
{\bf Output:} $x_l$.
\end{algorithm2e}

\begin{theorem}[Privacy and utility of Algorithm~\ref{alg: phased ERM without regularization} - Informal]
\label{thm: phased erm without reg}
Let $\eps \leq 10$, $n^{1 - o(1)} \gtrsim \frac{\log(n/\delta)}{\eps}$, $\beta \leq (L/R) \sqrt{d m n \eps}$, and $m \lesssim \text{poly}(n)$. Then, Algorithm~\ref{alg: phased ERM without regularization} is $(\eps, \delta)$-user-level DP. %
Further, \[
\expec F(x_l) - F^* \leq LR \cdot \widetilde{O}\left(\frac{1}{\sqrt{n m \eps}} + \frac{\sqrt{d \log(1/\delta)}}{\sqrt{n} \eps \sqrt{m}} \right).
\]
The gradient complexity of Algorithm~\ref{alg: phased ERM without regularization} is $\leq nm$. 
\end{theorem}

\begin{remark}[State-of-the-art excess risk in linear time, without the restrictive assumptions]
    Under the assumptions that $\beta < (L\eps^3/R) \sqrt{n^3/md^3}$ and $m \leq d/\eps^2 \leq n$, \cite{bassily2023user} gave a linear-time algorithm with similar excess risk to~Algorithm~\ref{alg: phased ERM without regularization}. However, their assumptions are very restrictive in practice: For example, in the canonical regime $n\approx d$, their assumption on $\beta$ rules out essentially every (non-linear) loss function. By contrast, our result holds even if the smoothness parameter is huge ($\beta \approx \sqrt{nmd}$) and we only require a logarithmic number of users. Thus, our algorithm and result is applicable to many practical ML problems.  
\end{remark}

To prove that~Algorithm~\ref{alg: phased ERM without regularization} is private, we essentially argue that for any phase $i$, the $\ell_2$-sensitivity of $\tilde{x}_i$ is upper bounded by $\widetilde{O}(\tau_i/C)$ with probability at least $1 - \delta/2$. The argument goes roughly as follows: First, the Laplace noise added to $s_i(\tau_i)$ ensures that $s_i(\tau_i)$ is $\eps/4$-user-level DP. Now, it suffices to assume $\widehat{s}_i(\tau_i) \geq 4C/5$, since otherwise the algorithm halts and outputs $0$ independently of the data. Next, conditional on the high probability event that the Laplace noise is smaller than $\widetilde{O}(1/\eps)$, we know that $\widehat{s}_i(\tau_i) \geq 4C/5 \implies s_i(\tau_i) \geq 2C/3$ with high probability by our choice of $C$. In this case, an argument along the lines of \cite[Lemma 3.5]{asi2023user} shows that $\tilde{x}_i$ has sensitivity bounded by $\widetilde{O}(\tau_i/C)$ with probability at least $1 - \delta/2$. See Appendix~\ref{app: proof of linear time thm} for the detailed proof. 

\vspace{.2cm}
To prove the excess risk bound in~Algorithm~\ref{alg: phased ERM without regularization}, the key step is to show that if the data is i.i.d., then with high probability, no points are removed from $\{\tilde{x}_{i,j}\}_{j=1}^C$ during the outlier-removal phase (i.e. $|\calS_i| = C$). If $|\calS_i| = C$ holds, then we can use the convergence guarantees of SGD and the localization arguments in~\cite{fkt20} to establish the excess risk guarantee. In order to prove that $|\calS_i| = C$ with high probability, we need the following novel \textit{stability} lemma:
\begin{lemma}
\label{lemma 2}
Assume $f(\cdot, z)$ is convex, $L$-Lipschitz, and $\beta$-smooth on $\XX$ with $\eta \leq 1/\beta$. Let $\tilde{x} \gets SGD(D, \eta, T, x_0)$ and $\tilde{y} \gets SGD(D', \eta, T, x_0)$ be two independent runs of projected SGD, where 
$D, D' \sim P^N$ are i.i.d. Then, with probability at least $1 - \zeta$, we have \[
\|\tilde{x} - \tilde{y}\| \lesssim \eta L\sqrt{T \log(dT/\zeta)}.
\]

\end{lemma}

We prove Lemma~\ref{lemma 2} via induction on $t$, using non-expansiveness of gradient descent on smooth losses~\cite{hardt16}, subgaussian concentration bounds, and a union bound. 

\vspace{.2cm}
Lemma~\ref{lemma 2} implies that if the data is i.i.d., then the following events hold with high probability: $\|\tilde{x}_{i,j} - \tilde{x}_{i,j'}\| \le \tau_i$ for all $j, j' \in [C_i]$ and hence $s_i(\tau_i) = C$. Further, conditional on $s_i(\tau_i) = C$, we know that $\widehat{s}_i(\tau_i) \geq 4C/5$ with high probability, so that the algorithm does not halt. Also, $\|\tilde{x}_{i,j} - \tilde{x}_{i,j'}\| \le \tau_i$ for all $j, j'$ implies $p_{i,j} = 1$ for all $j$ and hence $|\calS_i| = C$ for all $j$. The detailed excess risk proof is provided in Appendix~\ref{app: proof of linear time thm}.
\paragraph{Challenges of getting optimal excess risk in linear time:} In the item-level DP setting, there are several (nearly) linear time algorithms that achieve optimal excess risk for smooth DP SCO under mild smoothness conditions, such as snowball SGD~\cite{fkt20}, phased SGD~\cite{fkt20}, and phased ERM with output perturbation~\cite{zhang2022bring}. Extending these approaches into optimal nearly linear-time user-level DP algorithms is challenging. First, the user-level DP implementation of output perturbation in~\cite{ghazi2023user} is computationally inefficient. 
Second, snowball SGD relies on \textit{privacy amplification by iteration}, which does not extend nicely to the user-level DP case due to instability of the outlier-detection procedure in~\cite{asi2023user}. Specifically, since amplification by iteration intermediate only provides DP for the last iterate $x_T$ but not the intermediate iterates $x_t~(t < T)$, the sensitivity of the concentration score function is not $O(1)$, which impairs DP outlier-detection. A similar instability issue arises if one tries to naively extend phased SGD to be user-level DP by applying \cite{asi2023user} to user gradients. This issue motivates our Algorithm~\ref{alg: phased ERM without regularization}, which extends phased SGD in an alternative way: by applying outlier-detection/removal to the SGD \textit{iterates} instead of the gradients, we can control the sensitivity of the concentration score and thus prove that our algorithm is DP. However, since the bound in Lemma~\ref{lemma 2} scales polynomially with $T$ (and we believe this dependence on $T$ is necessary), Algorithm~\ref{alg: phased ERM without regularization} adds excessive noise and has suboptimal excess risk. 
We believe that obtaining optimal risk in linear time will require a fundamentally different user-level DP mean estimation procedure that does not suffer from the instability issue. 

\section{An optimal algorithm with $\approx (mn)^{9/8}$ gradient complexity for smooth losses}
\label{sec: accelerated phased ERM}
In this section, we provide an algorithm that achieves optimal excess risk using  $\approx (mn)^{9/8}$ stochastic gradient evaluations. Our~Algorithm~\ref{alg: accelerated phased ERM} is inspired by the item-level accelerated phased ERM algorithm of~\cite{kll21}. It applies iterative localization~\cite{fkt20} to the user-level DP accelerated minibatch SGD~Algorithm~\ref{alg: accelerated minibatch}. Algorithm~\ref{alg: accelerated minibatch} is a user-level DP variation of the accelerated minibatch SGD of~\cite{ghadimilan1,ghadimilan2}.

\begin{algorithm2e}
\caption{$\texttt{User-Level DP Accelerated Minibatch SGD}(\hat{F}_{i}, T_i, K_i, x_{i-1}, \tau, \eps, \delta)$}
\label{alg: accelerated minibatch}
Initialize $x_{i-1}^1 \gets x_{i-1}$\;
\For{$t=1,\cdots,T_i$}
{ 
Draw $K_i$ random users $D_i^t = \{Z_{i,j}^t\}_{j=1}^{K_i}$ from $D_i$ uniformly with replacement\; 
Set noisy threshold $\hat{\Delta}_i := \frac{4K_i}{5} + \xi_i$, where $\xi_i \sim \Lap\left(\frac{8}{\eps}\right)$\;
Let $q_t(Z) := \frac{1}{m} \sum_{z \in Z} \nabla f(x_{i-1}^{t}, z)$ for user $Z$\;
Compute the concentration score of $D_i^t$: 
\begin{align*}
    s_{i}^t(\tau):=\frac{1}{K_i}\sum_{Z, Z' \in D_i^t}\mathbf{1}(\|q_t(Z) -q_t(Z')\|\le 2\tau)
\end{align*} 
Let $\hat{s}_{i}^t(\tau)=s_{i}^t(\tau_i)+ v_i^t$, where $v_i^t \sim \Lap\left(\frac{16}{\eps}\right)$\;
\If{$\hat{s}_{i}^t(\tau)\ge \hat{\Delta}_i$}
{
$\mathcal{S}_i^t = \emptyset$\;
\For{Each User $Z \in D_i^t$}
{
Set $h_{i}^t(Z) =\sum_{Z' \in D_i^t} \ind(\|q_t(Z) - q_t(Z')\|\le2\tau)$\;
Add $Z$ to $\mathcal{S}_{i}^t$ with probability $p_{i}^t(Z):=
\begin{cases}
    0 & h_{i}^t(Z) < K_i /2 \\
    1 & h_{i}^t(Z) \ge 2K_i/3\\
    \frac{h_{i}^t(Z) - K_i/2}{K_i/6} & o.w.
\end{cases}  
$
}
$g_i^t = \frac{1}{|\mathcal{S}_i^t|} \sum_{Z \in \mathcal{S}_i^t} \nabla \widehat{F}(x_{i-1}^t, Z)$\;
$\widehat{g}_i^t = g_i^t + \zeta_i^t$, where $ \zeta_i^t \sim \mathcal{N}(0, \sigma_i^2)$ with $\sigma_i = \frac{1000 \tau \sqrt{T_i} \log(nde^{\eps}/\delta)}{\eps n_i}$\;
Do $1$ iteration of Accelerated Minibatch SGD (AC-SA)~\cite{ghadimilan1} on $\widehat{F}_i$, using gradient estimator $\widehat{g}_i^t + \lambda_i(x_{i-1}^t - x_{i-1})$ to obtain $x_{i-1}^{t+1}$.  

}
\Else
{ 
{\bf Halt; Return 0}
}
}
\textbf{Output} $x_{i-1}^{T_i}$.
\end{algorithm2e}

\begin{algorithm2e}
\caption{User-Level DP Accelerated Phased ERM with Outlier Gradient Removal}
\label{alg: accelerated phased ERM}
{\bf Input:} Dataset $\calD = (Z_1,\dots,Z_n) $, privacy parameters $(\epsilon,\delta)$, parameters $p, q, \lambda > 0$\;
Set $l = \lfloor \log_2(n) \rfloor $ and $\tau = O(L \log(ndm)/\sqrt{m})$, choose any initial point $x_0 \in \XX$\;
\For{$i=1,\cdots,l$}
{ 
Set $n_i= (1 - (1/2)^q) n /2^{iq}$, $\lambda_i=\lambda \cdot 2^{pi}$, $T_i = \widetilde{\Theta}(1 + \sqrt{\beta/\lambda_i})$, $K_i = 500 \log(n_i^2 m^2 e^{\eps}/\delta) \left(\frac{1}{\eps} + \frac{n_i \eps}{\sqrt{T_i \log(1/\delta)}}\right)$\; 
Draw disjoint users $D_{i}$ of size $n_i$ from $\calD$\;
Let $\hat{F}_{i}(x) := \frac{1}{n_i} \sum_{Z_{i,j} \in D_{i}} \hat{F}(x, Z_{i,j})$ + $\frac{\lambda_i}{2}\|x - x_{i-1}\|^2$, where $\hat{F}(x, Z_{i,j})$ is user $Z_{i,j}$'s empirical loss\;
$x_i \gets \texttt{User-Level DP Accelerated Minibatch SGD}(\hat{F}_{i}, T_i, K_i, x_{i-1}, \tau, \eps, \delta)$.
\;
}
\textbf{Output} $x_l$.
\end{algorithm2e}

Our~Algorithm~\ref{alg: accelerated minibatch} applies a DP outlier-removal procedure to the users' gradients in each iteration. 
We use \textit{Above Threshold}~\cite{dwork2014} to privatize the concentration scores  $s_i^{(t)}$ and determine whether or not most of the gradients of users in minibatch $D_i^t$ are $2\tau$-close to each other.
If $\widehat{s}_i^t \geq \widehat{\Delta}_i$, indicating that the gradients of users in $D_i^t$ are nearly $2\tau$-concentrated, then we proceed with outlier removal in lines 8-12. We then invoke \textit{privacy amplification by subsampling}~\cite{balle18} and the \textit{advanced composition theorem}~\cite{kairouz15} to privatize the average of the ``inlier'' gradients with additive Gaussian noise. By properly choosing algorithmic parameters, we obtain the following results, proved in Appendix~\ref{app: accelerated phased ERM}: 

\begin{theorem}[Privacy of Algorithm~\ref{alg: accelerated phased ERM}]
\label{thm: privacy of accelerated phased ERM}
    Let $\eps \leq 10$, $q > 0$ such that $n^{1-q} > \frac{100 \log(20nmde^{\eps}/\delta)}{\eps (1 - (1/2)^q)}$. 
    Then, Algorithm~\ref{alg: accelerated phased ERM} is $(\eps, \delta)$-DP.
\end{theorem}

\begin{theorem}[Utility \& runtime of Algorithm~\ref{alg: accelerated phased ERM} - Informal]
\label{thm: utility of accelerated phased ERM}
Let 
$\eps \le 10$ and $\delta < 1/(mn)$. Then, 
    Algorithm~\ref{alg: accelerated phased ERM} yields optimal excess risk:
    \[
    \expec F(x_l) - F^* \leq LR \cdot \widetilde{O}\left(\frac{1}{\sqrt{mn}} + \frac{\sqrt{d \log(1/\delta)}}{\eps n \sqrt{m}} \right).
    \]
    The gradient complexity of this algorithm is upper bounded by \[
   mn\left(1 + \eps\left(\frac{\beta R}{L}\right)^{1/4}\left((mn)^{1/8} \wedge \left(\frac{\eps^2 n^2 m}{d} \right)^{1/8} \right) \right) + \sqrt{\frac{\beta R}{L}}\left(\frac{n^{1/4} m^{5/4}}{\eps} + \left(\frac{n^{1/2}m^{5/4}}{d^{1/4} \eps^{1/2}} \right)\right).
    \]
    
\end{theorem}

If $n = d$, $\eps = 1$, and $\beta R = L$ then the gradient complexity bound in~\cref{thm: utility of accelerated phased ERM} simplifies to $(mn)^{9/8} + n^{1/4} m^{5/4}$. Typically, $n^7 \ge m$, so that the dominant term in this bound is $(mn)^{9/8}$. 

\begin{remark}[State-of-the-art runtime]
\vspace{.1cm}
The gradient complexity bound in~\cref{thm: utility of accelerated phased ERM} is \textit{superior to the runtime bounds of all existing near-optimal algorithms by polynomial factors} in $n,m$, and $d$~\cite{bassily2023user,ghazi2023user,asi2023user}. Note that while the $mn^{3/2}$ gradient complexity bound of~\cite{bassily2023user} may \textit{appear} to be better than $\beta^{1/4}(nm)^{9/8}$ in certain parameter regimes (e.g. $m > n^3$ or $\beta \gg nm$), this is not the case: the result of~\cite{bassily2023user} requires $m < n$ and $\beta < \sqrt{n/m}$.
\end{remark} 

\begin{remark}[Mild assumptions]
Note that~\cref{thm: privacy of accelerated phased ERM,thm: utility of accelerated phased ERM} do not require any bound on the smoothness parameter $\beta$, and only require the number of users to grow logarithmically: $n^{1- o(1)} \ge \widetilde{\Omega}(1/\eps)$. Contrast this with the results of previous works (e.g. \cite{bassily2023user}).
\end{remark}

A challenge in proving~\cref{thm: utility of accelerated phased ERM} is getting a tight bound on the variance of the the noisy minibatch stochastic gradients $\widehat{g}_i^t$ that are used in~Algorithm~\ref{alg: accelerated minibatch} (lines 12-14). Conditional on $\calS_i^t = D_i^t$, it is easy to obtain a variance bound of the form $\expec\|\widehat{g}_i^t - \nabla \widehat{F}_i(x_i^t)\|^2 \lesssim d \sigma_i^2 + \frac{L^2}{K_i}$, since we are sampling $K_i$ users uniformly at random. However, this bound is too weak to obtain~\cref{thm: utility of accelerated phased ERM}, since it does not scale with $m$. To prove~\cref{thm: utility of accelerated phased ERM}, we need the following stronger result:

\begin{lemma}[Variance Bound for~Algorithm~\ref{alg: accelerated minibatch}]
\label{lem: variance bound}
Let $\delta \leq 1/(nm), \eps \lesssim 1$. Denote $\widetilde{F}_i(x) := 
\frac{1}{n_i} \sum_{Z_{i,j} \in D_{i}} \hat{F}(x, Z_{i,j})
$.
Then, conditional on $\calS_i^t = D_i^t$ for all $i \in [l], t \in [T_i]$, we have \[
\expec\| g_i^t - \nabla \widetilde{F}_i(x_{i-1}^t) \|^2 \lesssim \frac{L^2 \log(ndm)}{Km}
\]
for all $i \in [l], t \in [T_i]$, where the expectation is over both  
the random i.i.d. draw of $\DD = (Z_1, \ldots, Z_n) \sim P^{nm}$ and the randomness in Algorithm~\ref{alg: accelerated phased ERM}. 
\end{lemma} 

The difficulty in proving Lemma~\ref{lem: variance bound} comes from the fact that the iterates $x_i^t$ and the data $\DD$ are not independent. To overcome this difficulty, we use the \textit{stability of user-level DP}~\cite{bassily2023user} to argue that for all $Z \in D_i$, $\nabla \hf(x_{i-1}^t, Z)$ is $\approx L/\sqrt{m}$-close to $\nabla F(x_{i-1}^t)$ with high probability, since $x_{i-1}^t$ is user-level DP. A detailed proof is given in Appendix~\ref{app: accelerated phased ERM}.  

\begin{remark}[Strongly convex losses: Optimal excess risk with state-of-the-art runtime]
If $f(\cdot, z)$ is $\mu$-strongly convex, then Algorithm~\ref{alg: accelerated phased ERM} can be combined with the meta-algorithm of \cite[Section 5.1]{fkt20} to obtain optimal excess risk \[
\frac{L^2}{\mu} \cdot \widetilde{O}\left(\frac{1}{nm} + \frac{d \ln(1/\delta)}{\eps^2 n^2 m} \right)
\]
with the same gradient complexity stated in~\cref{thm: utility of accelerated phased ERM}. This improves over the previous state-of-the-art gradient complexity $\approx \beta (mn)^{3/2}$
of~\cite{asi2023user}. 
\end{remark}

\section{An optimal algorithm with subquadratic gradient complexity for non-smooth losses}
\label{sec: nonsmooth}
In this section, we extend our accelerated algorithm from the previous section to non-smooth loss functions. To accomplish this with minimal computational cost, we apply \textit{randomized (convolution) smoothing}~\cite{yns11,dbw12} to approximate non-smooth $f$ by a $\beta$-smooth $\tilde{f}$. We can then apply~Algorithm~\ref{alg: accelerated phased ERM} to $\tilde{f}$. Since convolution smoothing is by now a standard optimization technique, we defer the details and proof to Appendix~\ref{app: nonsmooth}.

\begin{theorem}[Privacy and utility of smoothed~Algorithm~\ref{alg: accelerated phased ERM} for non-smooth loss - informal]
\label{thm: nonsmooth}
Let $\eps \le 10$, $\delta < 1/(mn)$, and $q > 0$ such that $n^{1-q} > \frac{100 \log(20nmde^{\eps}/\delta)}{\eps (1 - (1/2)^q)}$. 
Then, applying Algorithm~\ref{alg: accelerated phased ERM} to the smooth approximation of $f$
yields optimal excess risk:
    \[
    \expec F(x_l) - F^* \leq LR \cdot \widetilde{O}\left(\frac{1}{\sqrt{mn}} + \frac{\sqrt{d \log(1/\delta)}}{\eps n \sqrt{m}} \right).
    \]
    The gradient complexity of this algorithm is upper bounded by \[
   mn\left(1 +n^{3/8}m^{1/4}\eps^{1/4}\right).
    \]
\end{theorem}

\begin{remark}[State-of-the-art gradient complexity]
The only previous polynomial-time algorithm that can achieve optimal excess risk for non-smooth loss functions is due to~\cite{asi2023user}. The algorithm of~\cite{asi2023user} required $(nm)^3 + (mn)^2\sqrt{d}$ gradient evaluations. Thus, the gradient complexity of the smoothed version of~Algorithm~\ref{alg: accelerated phased ERM} offers a \textit{significant improvement over the previous state-of-the-art}. For example, if $\eps = 1$, then our algorithm is faster than the previous state-of-the-art by a multiplicative factor of at least $n^{13/8} m^{7/4}$.
\end{remark}

\section{Conclusion}
\label{sec: conclusion}
In this paper, we developed new user-level DP algorithms with improved runtime and excess risk guarantees for stochastic convex optimization without the restrictive assumptions made in prior works. Our accelerated~Algorithm~\ref{alg: accelerated phased ERM} achieves optimal excess risk for both smooth and non-smooth loss functions, with significantly smaller computational cost than the previous state-of-the-art. Our linear-time~Algorithm~\ref{alg: phased ERM without regularization} achieves state-of-the-art excess risk under much milder, more practical assumptions than existing linear-time approaches. 

\vspace{.2cm}
Our work paves the way for several intriguing future research directions. First, the question of whether there exists a linear-time algorithm that can attain the user-level DP lower bound for smooth losses remains open. In light of our improved gradient complexity bound ($\approx (nm)^{9/8}$), we are now optimistic that the answer to this question is ``yes.'' We believe that our novel techniques will be key to the development of an optimal linear-time algorithm. Specifically, utilizing Lemma~\ref{lemma 2} to apply outlier removal to the iterates instead of the gradients appears to be pivotal. Second, the study of user-level DP SCO has been largely limited to approximate $(\eps, \delta)$-DP. What rates are achievable under the stronger notion of pure $\eps$-user-level DP? Third, it would be useful to develop fast and optimal algorithms that are tailored to federated learning environments~\cite{mcmahan17,gao2024private}, where only a small number of users may be available to communicate with the server in each iteration. We hope our work inspires and guides further research in this exciting and practically important area. 

\section*{Acknowledgements}
AL’s research is supported by NSF grant 2023239 and the AFOSR award FA9550-21-1-0084. We thank the anonymous NeurIPS reviewers for their helpful feedback.

\newpage
\bibliographystyle{alpha}
\bibliography{ref}

\newcommand{\etalchar}[1]{$^{#1}$}
\begin{thebibliography}{CTW{\etalchar{+}}21}

\bibitem[AFKT21]{AsiFeKoTa21}
Hilal Asi, Vitaly Feldman, Tomer Koren, and Kunal Talwar.
\newblock Private stochastic convex optimization: Optimal rates in {$\ell_1$} geometry.
\newblock In {\em ICML}, 2021.

\bibitem[AL24]{asi2023user}
Hilal Asi and Daogao Liu.
\newblock User-level differentially private stochastic convex optimization: Efficient algorithms with optimal rates.
\newblock In {\em International Conference on Artificial Intelligence and Statistics}, pages 4240--4248. PMLR, 2024.

\bibitem[ALT24]{asi2024private}
Hilal Asi, Daogao Liu, and Kevin Tian.
\newblock Private stochastic convex optimization with heavy tails: Near-optimality from simple reductions.
\newblock {\em arXiv preprint arXiv:2406.02789}, 2024.

\bibitem[BBG18]{balle18}
Borja Balle, Gilles Barthe, and Marco Gaboardi.
\newblock Privacy amplification by subsampling: Tight analyses via couplings and divergences.
\newblock In S.~Bengio, H.~Wallach, H.~Larochelle, K.~Grauman, N.~Cesa-Bianchi, and R.~Garnett, editors, {\em Advances in Neural Information Processing Systems}, volume~31. Curran Associates, Inc., 2018.

\bibitem[BFTT19]{bft19}
Raef Bassily, Vitaly Feldman, Kunal Talwar, and Abhradeep Thakurta.
\newblock Private stochastic convex optimization with optimal rates.
\newblock In {\em Advances in Neural Information Processing Systems}, volume~32, 2019.

\bibitem[BS23]{bassily2023user}
Raef Bassily and Ziteng Sun.
\newblock User-level private stochastic convex optimization with optimal rates.
\newblock In Andreas Krause, Emma Brunskill, Kyunghyun Cho, Barbara Engelhardt, Sivan Sabato, and Jonathan Scarlett, editors, {\em Proceedings of the 40th International Conference on Machine Learning}, volume 202 of {\em Proceedings of Machine Learning Research}, pages 1838--1851. PMLR, 23--29 Jul 2023.

\bibitem[CTW{\etalchar{+}}21]{carlini2021extracting}
Nicholas Carlini, Florian Tramer, Eric Wallace, Matthew Jagielski, Ariel Herbert-Voss, Katherine Lee, Adam Roberts, Tom~B Brown, Dawn Song, Ulfar Erlingsson, et~al.
\newblock Extracting training data from large language models.
\newblock In {\em USENIX Security Symposium}, volume~6, pages 2633--2650, 2021.

\bibitem[DBW12]{dbw12}
John~C Duchi, Peter~L Bartlett, and Martin~J Wainwright.
\newblock Randomized smoothing for stochastic optimization.
\newblock {\em SIAM Journal on Optimization}, 22(2):674--701, 2012.

\bibitem[DMNS06]{dwork2006calibrating}
Cynthia Dwork, Frank McSherry, Kobbi Nissim, and Adam Smith.
\newblock Calibrating noise to sensitivity in private data analysis.
\newblock In {\em Theory of cryptography conference}, pages 265--284. Springer, 2006.

\bibitem[DR14]{dwork2014}
Cynthia Dwork and Aaron Roth.
\newblock {\em The Algorithmic Foundations of Differential Privacy}, volume~9.
\newblock Now Publishers, Inc., 2014.

\bibitem[FKT20]{fkt20}
Vitaly Feldman, Tomer Koren, and Kunal Talwar.
\newblock Private stochastic convex optimization: optimal rates in linear time.
\newblock In {\em Proceedings of the 52nd Annual ACM SIGACT Symposium on Theory of Computing}, pages 439--449, 2020.

\bibitem[GKK{\etalchar{+}}23]{ghazi2023user}
Badih Ghazi, Pritish Kamath, Ravi Kumar, Pasin Manurangsi, Raghu Meka, and Chiyuan Zhang.
\newblock On user-level private convex optimization.
\newblock In {\em International Conference on Machine Learning}, pages 11283--11299. PMLR, 2023.

\bibitem[GL12]{ghadimilan1}
Saeed Ghadimi and Guanghui Lan.
\newblock Optimal stochastic approximation algorithms for strongly convex stochastic composite optimization i: A generic algorithmic framework.
\newblock {\em SIAM Journal on Optimization}, 22(4):1469--1492, 2012.

\bibitem[GL13]{ghadimilan2}
Saeed Ghadimi and Guanghui Lan.
\newblock Optimal stochastic approximation algorithms for strongly convex stochastic composite optimization, ii: Shrinking procedures and optimal algorithms.
\newblock {\em SIAM Journal on Optimization}, 23(4):2061--2089, 2013.

\bibitem[GLZW24]{gao2024private}
Changyu Gao, Andrew Lowy, Xingyu Zhou, and Stephen~J Wright.
\newblock Private heterogeneous federated learning without a trusted server revisited: Error-optimal and communication-efficient algorithms for convex losses.
\newblock {\em arXiv preprint arXiv:2407.09690}, 2024.

\bibitem[HRS16]{hardt16}
Moritz Hardt, Ben Recht, and Yoram Singer.
\newblock Train faster, generalize better: Stability of stochastic gradient descent.
\newblock In Maria~Florina Balcan and Kilian~Q. Weinberger, editors, {\em Proceedings of The 33rd International Conference on Machine Learning}, volume~48 of {\em Proceedings of Machine Learning Research}, pages 1225--1234, New York, New York, USA, 20--22 Jun 2016. PMLR.

\bibitem[JNG{\etalchar{+}}19]{jin2019short}
Chi Jin, Praneeth Netrapalli, Rong Ge, Sham~M Kakade, and Michael~I Jordan.
\newblock A short note on concentration inequalities for random vectors with subgaussian norm.
\newblock {\em arXiv preprint arXiv:1902.03736}, 2019.

\bibitem[KLL21]{kll21}
Janardhan Kulkarni, Yin~Tat Lee, and Daogao Liu.
\newblock Private non-smooth erm and sco in subquadratic steps.
\newblock {\em Advances in Neural Information Processing Systems}, 34:4053--4064, 2021.

\bibitem[KOV15]{kairouz15}
Peter Kairouz, Sewoong Oh, and Pramod Viswanath.
\newblock The composition theorem for differential privacy, 2015.

\bibitem[LJCJ17]{lei17}
Lihua Lei, Cheng Ju, Jianbo Chen, and Michael~I Jordan.
\newblock Non-convex finite-sum optimization via scsg methods.
\newblock In {\em Proceedings of the 31st International Conference on Neural Information Processing Systems}, pages 2345--2355, 2017.

\bibitem[LLL{\etalchar{+}}24a]{li2024analyzing}
Zhuohang Li, Andrew Lowy, Jing Liu, Toshiaki Koike-Akino, Kieran Parsons, Bradley Malin, and Ye~Wang.
\newblock Analyzing inference privacy risks through gradients in machine learning.
\newblock {\em arXiv preprint arXiv:2408.16913}, 2024.

\bibitem[LLL{\etalchar{+}}24b]{lowy2024does}
Andrew Lowy, Zhuohang Li, Jing Liu, Toshiaki Koike-Akino, Kieran Parsons, and Ye~Wang.
\newblock Why does differential privacy with large epsilon defend against practical membership inference attacks?
\newblock {\em arXiv preprint arXiv:2402.09540}, 2024.

\bibitem[LR22]{lowy2023largelip}
Andrew Lowy and Meisam Razaviyayn.
\newblock Private stochastic optimization with large worst-case lipschitz parameter, 2022.

\bibitem[LSA{\etalchar{+}}21]{levy2021learning}
Daniel Levy, Ziteng Sun, Kareem Amin, Satyen Kale, Alex Kulesza, Mehryar Mohri, and Ananda~Theertha Suresh.
\newblock Learning with user-level privacy.
\newblock {\em Advances in Neural Information Processing Systems}, 34:12466--12479, 2021.

\bibitem[LUW24]{lowymake}
Andrew Lowy, Jonathan Ullman, and Stephen Wright.
\newblock How to make the gradients small privately: Improved rates for differentially private non-convex optimization.
\newblock In {\em Forty-first International Conference on Machine Learning}, 2024.

\bibitem[McS09]{mcsherry2009privacy}
Frank~D McSherry.
\newblock Privacy integrated queries: an extensible platform for privacy-preserving data analysis.
\newblock In {\em Proceedings of the 2009 ACM SIGMOD International Conference on Management of data}, pages 19--30, 2009.

\bibitem[MRTZ18]{mcmahan17}
Brendan McMahan, Daniel Ramage, Kunal Talwar, and Li~Zhang.
\newblock Learning differentially private recurrent language models.
\newblock In {\em International Conference on Learning Representations (ICLR)}, 2018.

\bibitem[SSSS17]{shokri2017membership}
Reza Shokri, Marco Stronati, Congzheng Song, and Vitaly Shmatikov.
\newblock Membership inference attacks against machine learning models.
\newblock In {\em 2017 IEEE symposium on security and privacy (SP)}, pages 3--18. IEEE, 2017.

\bibitem[SSSSS09]{shalev2009stochastic}
Shai Shalev-Shwartz, Ohad Shamir, Nathan Srebro, and Karthik Sridharan.
\newblock Stochastic convex optimization.
\newblock In {\em COLT}, volume~2, page~5, 2009.

\bibitem[TCK{\etalchar{+}}22]{tsfadia2022friendlycore}
Eliad Tsfadia, Edith Cohen, Haim Kaplan, Yishay Mansour, and Uri Stemmer.
\newblock Friendlycore: Practical differentially private aggregation.
\newblock In {\em International Conference on Machine Learning}, pages 21828--21863. PMLR, 2022.

\bibitem[Ull17]{ullman2017}
Jonathan Ullman.
\newblock {CS7880:} rigorous approaches to data privacy, 2017.

\bibitem[XZ24]{Xu2024}
Zheng Xu and Yanxiang Zhang.
\newblock Advances in private training for production on-device language models.
\newblock \url{https://research.google/blog/advances-in-private-training-for-production-on-device-language-models/}, 2024.
\newblock Google Research Blog.

\bibitem[YNS12]{yns11}
Farzad Yousefian, Angelia Nedi{\'c}, and Uday~V Shanbhag.
\newblock On stochastic gradient and subgradient methods with adaptive steplength sequences.
\newblock {\em Automatica}, 48(1):56--67, 2012.

\bibitem[ZTOH22]{zhang2022bring}
Liang Zhang, Kiran~K Thekumparampil, Sewoong Oh, and Niao He.
\newblock Bring your own algorithm for optimal differentially private stochastic minimax optimization.
\newblock {\em Advances in Neural Information Processing Systems}, 35:35174--35187, 2022.

\end{thebibliography}

\newpage
\appendix

\section*{Appendix}

\section{More Preliminaries}
\label{app: prelim}
\subsection{Tools from Differential Privacy}
\paragraph{Additive Noise Mechanisms}

Additive noise mechanisms privatize a query by adding noise to its output, with the scale of the noise calibrated to the \textit{sensitivity} of the query. 
\begin{definition}[Sensitivity]
Given a function $q: \ZZ^N \to \R^k$ and a norm $\| \cdot \|_p$ on $\R^k$, the $\ell_p$-\textit{sensitivity} of $q$ is defined as \[
\sup_{\DD \sim \DD'} \|q(\DD) - q(\DD') \|_p,
\]
where the supremum is taken over all pairs of datasets that differ in one user's data. 
\end{definition}

\begin{definition}[Laplace Distribution]
We say $X\sim\Lap(b)$ if the density of $X$ is $f(X=x)=\frac{1}{2b}\exp(-\frac{|x|}{b})$.  
\end{definition}

\begin{definition}[Laplace Mechanism]
Let $\eps > 0$. Given a function $q: \ZZ^N \to \R^k$ on $\R^k$ with $\ell_1$-sensitivity $\Delta$, the \textit{Laplace Mechanism} $\mathcal{M}$ is defined by \[
\mathcal{M}(\DD) := q(\DD) + (Y_1, \ldots, Y_k), 
\]
where $\{Y_i\}_{i=1}^k$ are i.i.d., $Y_i \sim \Lap\left(\frac{\Delta}{\eps}\right)$. 
\end{definition}

\begin{lemma}[Privacy of Laplace Mechanism~\cite{dwork2014}]
\label{lem: laplace mech}
The Laplace Mechanism is $\eps$-DP. 
\end{lemma}

\begin{definition}[Gaussian Mechanism]
Let $\eps > 0$, $\delta \in (0, 1)$. Given a function $q: \ZZ^N \to \R^k$ with $\ell_2$-sensitivity $\Delta$, the \textit{Gaussian Mechanism} $\mathcal{M}$ is defined by \[
\mathcal{M}(\DD) := q(\DD) + G 
\]
where $G \sim \mathcal{N}_k\left(0, \sigma^2 \mathbf{I}_k\right)$ and $\sigma^2 = \frac{2 \Delta^2 \log(2/\delta)}{\eps^2}$. 
\end{definition}

\begin{lemma}[Privavcy of Gaussian Mechanism~\cite{dwork2014}]
\label{lem: gauss mech}
The Laplace Mechanism is $(\eps, \delta)$-DP. 
\end{lemma}

\paragraph{Advanced Composition}
If we adaptively query a data set $T$ times, then the privacy guarantees of the $T$-th query is still DP and the privacy parameters degrade gracefully:
\begin{lemma}[Advanced Composition Theorem~\cite{dwork2014}]
\label{thm: advanced composition}
Let $\eps \geq\ 0, \delta, \delta' \in [0, 1)$. Assume $\mathcal{A}_1, \cdots, \mathcal{A}_T$, with $\Alg_t: \ZZ^n \times \XX \to \XX$, are each $(\eps, \delta)$-DP ~$\forall t = 1, \cdots, T$. Then, the adaptive composition $\mathcal{A}(\DD) := \mathcal{A}_T(\DD, \Alg_{T-1}(\DD, \Alg_{T-2}(X, \cdots)))$ is $(\eps', T\delta + \delta')$-DP for \[
\eps' = \sqrt{2T \ln(1/\delta')} \eps + T\eps(e^{\eps} - 1).\]
\end{lemma}

\paragraph{Privacy Amplification by Subsampling}
\begin{lemma}[\cite{ullman2017}]
\label{lem: amp by subsampling}
Let $\mathcal{M}: \ZZ^{M} \to \XX$ be $(\eps, \delta)$-DP. Let $\mathcal{M}': \ZZ^N \to \XX$ that first selects a random subsample $\DD'$ of size $M$ from the data set $\DD \in \ZZ^N$ and then outputs $\mathcal{M}(\DD')$. Then, $\mathcal{M}'$ is $(\eps', \delta')$-DP, where $\eps' = \frac{(e^{\eps} - 1)M}{N}$ and $\delta' = \frac{\delta M}{N}$. 
\end{lemma}

\paragraph{AboveThreshold:}
AboveThreshold algorithm~\cite{dwork2014} which is a key tool in differential privacy to identify whether there is a query $q_i: \calZ \to \R$ in a stream of queries $q_1,\dots,q_T$ that is above a certain threshold $\Delta$. 
The $\AboTh$ Algorithm~\ref{alg:mean_est_with_AT} has the following guarantees:

\begin{algorithm2e}
\caption{$\AboTh$}
\label{alg:mean_est_with_AT}
{\bf Input:} Dataset $\calD = (Z_1,\dots,Z_n) $, threshold $\Delta \in \R$, privacy parameter $\epsilon$, sequence of $T$ queries $q_1,\cdots,q_T : \calZ^n \to \R$, each with $\ell_1$-sensitivity $1$\;
Let $\hat{\Delta}:= \Delta + \Lap(\frac{2}{\epsilon})$\;
\For{$t=1$ to $T$}
{
Receive a new query $q_t: \calZ^n \to \R$ \;
Sample $\nu_i \sim \Lap(\frac{4}{\epsilon})$\;
\If{$q_t(\calD)+\nu_i \geq \hat{\Delta}$}
{
{\bf Output:} $a_t=\top$\;
{\bf Halt}\;
}
\Else{
{\bf Output:} $a_t=\bot$\;
}
}
\end{algorithm2e}

\begin{lemma}[\cite{dwork2014}, Theorem 3.24]
\label{thm:Above_Threshold}
    Let $\gamma > 0$ and $\alpha=\frac{8\log(2T/\gamma)}{\epsilon}$, $k \in [T+1]$. $\AboTh$ is $(\epsilon,0)$-DP. Moreover, with probability at least $1 - \gamma$, for all $t \le k$, we have:
    \begin{itemize}
        \item if $a_t = \top$, then $q_t(\DD) \geq \Delta - \alpha$; and
        \item if $a_t = \bot$, then $q_t(\DD) \le \Delta + \alpha$.
    \end{itemize} %
\end{lemma}

\subsection{SubGaussian and Norm-SubGuassian Random Vectors}
\begin{definition}
Let $\zeta > 0$. We say a random vector $X$ is \textit{SubGaussian} ($\mathrm{SG}(\zeta)$) with parameter $\zeta$ if %
$\E[e^{\langle v,X-\E X\rangle}]\le e^{\|v\|^2\zeta^2/2}$ for any $v\in \R^d$.
Random vector $X\in \R^d$ is \textit{Norm-SubGaussian} with parameter $\zeta$ ($\nSG(\zeta)$) if 
$\mathbb{P}[\|X-\E X\|\ge t]\le 2e^{-\frac{t^2}{2\zeta^2}}$ for all $t > 0$.
\end{definition}

\begin{theorem}[Hoeffding-type inequality for norm-subGaussian, \cite{jin2019short}]
\label{thm:hoeffding_nSG}
    Let $X_1,\cdots,X_k\in\R^d$ be random vectors, and let $\F_i=\sigma(x_1,\cdot,x_i)$ for $i\in[k]$ be the corresponding filtration.
    Suppose for each $i\in[k]$, $X_i\mid \F_{i-1}$ is zero-mean $\nSG(\zeta_i)$. Then, there exists an absolute constant $c>0$, for any $\gamma>0$,
    \begin{align*}
        \mathbb{P}\left[\left\|\sum_{i\in[k]}X_i\right\|\ge c\sqrt{\log (d/\gamma)\sum_{i\in[k]}\zeta_i^2}\right]\le \gamma.
    \end{align*}
\end{theorem}

\begin{lemma}[\cite{jin2019short}]
\label{lm:inner_product_nSG}
There exists an absolute constant $c$, such that if $X$ is $\nSG(\zeta)$, then for any fixed unit vector $v\in \R^d$, 
$\langle v,X\rangle$ is $c\zeta$ norm-SubGaussian.
\end{lemma}

\section{Proof of~\cref{thm: phased erm without reg}}
\label{app: proof of linear time thm}

\begin{theorem}[Formal statement of \cref{thm: phased erm without reg}]
\label{thm: phased erm without reg - appendix}
Suppose $n^{1-q} \geq (100/(1 -  1/2^q)) \log(n/\delta)/\eps$ for some small $q > 0$, and $m \leq n^J$ for some large $J > 0$. Choose $p = J + 3/2$ and 
$\eta = R/(L \sqrt{d m n \eps})$. 
in Algorithm~\ref{alg: phased ERM without regularization}. Then, Algorithm~\ref{alg: phased ERM without regularization} is $(\eps, \delta)$-user-level DP and achieves excess risk \[
\expec F(x_l) - F^* \leq LR \cdot \widetilde{O}\left(\frac{1}{\sqrt{n m \eps}} + \frac{\sqrt{d \log(1/\delta)}}{\sqrt{n} \eps \sqrt{m}} \right),
\]
using $nm$ gradient evaluations, provided $\beta \leq (L/R) \sqrt{d m n \eps}$. 
\end{theorem}

The gradient complexity is clear by inspection of the algorithm: The number of stochastic gradients computed during the algorithm is \[
\sum_{i=1}^l T_i C = \sum_{i=1}^l N_i m C = \sum_{i=1}^l n_i m \leq n m. 
\] 

Next, we will prove the privacy statement in~\cref{thm: phased erm without reg - appendix}. 
The following lemma ensures that if the Laplace noise added in~Algorithm~\ref{alg: phased ERM without regularization} is sufficiently small and outlier detection succeeds, then the sensitivity of $\tilde{x}_i$ is $\widetilde{O}(\tau_i/C)$ with high probability.  
\begin{lemma}{\cite[Slight modification of Lemma 3.5]{asi2023user}}
\label{lemma 3.5}
Let $i \in [l]$ and $\zeta > 0$. Suppose $\DD_i$ and $\DD_i'$ differ in the data of one user and we are in phase $i$ of Algorithm~\ref{alg: phased ERM without regularization}. Let $E_i$ be the event that the Laplace noise added to the concentration score $s_i(\tau_i)$ for $\DD_i$ has absolute value less than $2C/15$ and define $E'_i$ similarly for data $\DD_i'$. Denote $a_i := \ind(\widehat{s}_i(\tau_i) \geq 4C/5)$ and $a'_i := \ind(\widehat{s}'_i(\tau_i) \geq 4C/5)$, where $\widehat{s}_i(\tau_i)$ and $\widehat{s}'_i(\tau_i)$ are the noisy concentration scores that we get when running phase $i$ of Algorithm~\ref{alg: phased ERM without regularization} on neighboring $\DD_i$ and $\DD'_i$, respectively. Then, conditional on $a_i = a_i'$ and $E_i \bigcap E_i'$, there is a coupling $\Gamma_i$ over $\tilde{x}_i$ and $\tilde{x}_i'$ such that for $(y_i, y'_i)$ drawn from $\Gamma_i$, we have 
\[
\|y_i - y'_i\| \lesssim \frac{\tau_i \log(1/\zeta)}{C}
\]
with probability at least $1 - \zeta$. 
\end{lemma}

With Lemma~\ref{lemma 3.5} in hand, we proceed to prove that~Algorithm~\ref{alg: phased ERM without regularization} is $(\eps, \delta)$-user-level DP: 
\begin{proof}[Proof of~\cref{thm: phased erm without reg - appendix} - Privacy]
\textbf{Privacy:}  Since the $\{D_i\}_{i=1}^l$ are disjoint, parallel composition of DP~\cite{mcsherry2009privacy} implies that it suffices to prove that phase $i$ is $(\eps, \delta)$-user-level-DP for any fixed $i$ and fixed $x_{i-1}$. To that end, let $\DD$ and $\DD'$ be adjacent datasets differing in the data of one user, say $Z_{i,1} \neq Z_{i,1}'$ without loss of generality. We will show that the outputs of phase $i$ when run on $\DD$ and $\DD'$, $x_i := x_i(\DD)$ and $x'_i := x_i(\DD')$ respectively, are $(\eps, \delta)$-indistinguishable. 

Let $E_i$ be the event that the Laplace noise added in phase $i$ (for data set $\DD$) has absolute value less than $2C/15$ and define $E'_i$ analogously for data set $\DD'$. Note that $E_i$ and $E'_i$ are independent and $\mathbb{P}(E_i, E'_i) \geq 1 - \delta/10e^{\eps}$. Denote $\zeta := \delta/10e^{\eps}$. Let $a_i := \ind(\widehat{s}_i(\tau_i) \geq 4C/5)$ and $a'_i := \ind(\widehat{s}'_i(\tau_i) \geq 4C/5)$, where $\widehat{s}_i(\tau_i)$ and $\widehat{s}'_i(\tau_i)$ are the noisy concentration scores that we get when running phase $i$ of Algorithm~\ref{alg: phased ERM without regularization} on neighboring $\DD_i$ and $\DD'_i$, respectively. 
By Lemma~\ref{lemma 3.5} and our choice of $C$, we know that, conditional on $E_i \bigcap E'_i$ and on $a_i = a'_i$, 
there exists a coupling $\Gamma$ over $(\tilde{x}_i, \tilde{x}'_i)$ such that for $(y_i, y'_i)$ drawn from $\Gamma$, we have \begin{equation}
\label{eq: sens bound lemma 3.5}
\|y_i - y'_i \| \lesssim \frac{\tau_i \log(1/\zeta)}{C}
\end{equation}
with probability at least $1 - \zeta$. 

Note that the sensitivity of $s_i$ is less than or equal to $2$. Thus, by the privacy guarantees of the Laplace mechanism (Lemma~\ref{lem: laplace mech}), we have \begin{equation}
    \pr(a_i = b) \leq e^{\eps/4} \pr(a'_i = b) 
\end{equation}
for any $b \in \{0,1\}$. Further, this implies \begin{equation}
\label{eq: a_i lap bound}
    \pr(a_i = b, E_i) \leq e^{\eps/4} \left[\pr(a'_i = b, E_i') + \zeta \right].
\end{equation}

By the bound~\eqref{eq: sens bound lemma 3.5}, the privacy guarantee of the Gaussian mechanism (Lemma~\ref{lem: gauss mech}), our choice of $\sigma_i$,  and independence of the Laplace and Gaussian noises that we add in~Algorithm~\ref{alg: phased ERM without regularization}, we have \begin{equation}
\label{eq: gauss noise}
\pr(x_i \in \calO \mid E_i, a_i = 1) \leq e^{\eps/4} \pr(x_i' \in \calO \mid E'_i, a_i' = 1) + \frac{\delta}{n}  + \zeta,
\end{equation}
for any event $\calO \subset \XX$.

Moreover, since the algorithm halts and returns $x_i = 0$ if $a_i = 0$, we know that \begin{equation}
\label{eq: halt}
  \pr(x_i \in \calO \mid E_i, a_i = 0) = \pr(x_i' \in \calO \mid E'_i, a_i' = 0)
\end{equation}
for any event $\calO \subset \XX$. 

Therefore, \begin{align*}
    \pr(x_i \in \calO) &= \pr(x_i \in \calO \mid E_i) \pr(E_i) + \pr(x_i \in \calO \mid E_i^c) \pr(E_i^c) \\
    &\leq  \pr(x_i \in \calO \mid E_i, a_i = 1) \pr(E_i, a_i = 1) + \pr(x_i \in \calO \mid E_i, a_i = 0)\pr(E_i, a_i = 0) + \zeta \\
    &\stackrel{(i)}{\le} e^{\eps/4}\pr(x_i' \in \calO \mid E'_i, a_i' = 1) e^{\eps/4} \left[\pr(E'_i, a_i' = 1) + \zeta \right] \\
    &\;\;\;\; + \pr(x_i' \in \calO \mid E'_i, a_i' = 0) e^{\eps/4} \left[\pr(E'_i, a_i' = 0) + \zeta \right] + \zeta \\
    &\leq e^{\eps/2} \pr(x_i' \in \calO, E'_i) + \zeta\left(2 e^{\eps/2} + 1 \right) \\
    &\leq e^{\eps} \pr(x_i' \in \calO) + \delta,
\end{align*}
where $(i)$ follows from inequalities~\eqref{eq: a_i lap bound}, \eqref{eq: gauss noise}, and \eqref{eq: halt}. Thus, $x_i$ is $(\eps, \delta)$-user-level-DP. This completes the privacy proof. 
\end{proof}

Next, we turn to the excess risk proof. The following lemma is immediate from \cite[Lemma 4.5]{fkt20}: 
\begin{lemma}
\label{lem: FKT Lem4.3}
Let $\eta_i \leq 1/\beta$. Then, for any $y \in \XX$ and all $i,j$, we have
\[
\mathbb{E}[F(\tilde{x}_{i,j}) - F(y)] \le \frac{\mathbb{E}\|y - x_{i-1}\|^2}{\eta_i T_i} + 
\eta_i L^2. 
\]

\end{lemma}

The next novel lemma is crucial in our analysis: 
\begin{lemma}[Re-statement of Lemma~\ref{lemma 2}]
Assume $f(\cdot, z)$ is convex, $L$-Lipschitz, and $\beta$-smooth on $\XX$ with $\eta \leq 1/\beta$. Let $\tilde{x} \gets SGD(D, \eta, T, x_0)$ and $\tilde{y} \gets SGD(D', \eta, T, x_0)$ be two independent runs of projected SGD, where 
$D, D' \sim P^N$ are i.i.d. Then, with probability at least $1 - \zeta$, we have \[
\|\tilde{x} - \tilde{y}\| \lesssim \eta L\sqrt{T \log(dT/\zeta)}.
\]

\end{lemma}

\begin{proof}
Let $g_t:=\nabla f(x_t, z_t)$ for $z_t$ drawn uniformly from $D$ without replacement and $g_t':= \nabla f(y_t, z'_t)$ for $z'_t$ drawn uniformly from $D'$ without replacement. Let $F(x) := \expec_{z \sim P}[f(x,z)]$. 

We will prove that $\|x_t - y_t\| \lesssim \eta L\sqrt{T \log(dT/\zeta)}$ with probability at least $1 - \zeta/t$ for all $t \in [T]$. Note that this implies the lemma. We proceed by induction. The base case, when $t=0$, is trivially true since $x_0 = y_0$. For the inductive hypothesis, suppose there is an absolute constant $c > 0$ such that with probability at least $1-t\zeta/T$, we have 
\begin{align*}
    \|x_{i}-y_i\|\le  c  \eta L\sqrt{i\cdot \log(dT/\zeta)} + 2\eta L,
\end{align*}
$\forall i \le t$. Then, for the inductive step, we have by non-expansiveness of projection onto convex sets, that
\begin{align}
\label{eq: thingy}
    \|x_{t+1}-y_{t+1}\|^2 \le &~ \|x_t-\eta g_t-(y_t-\eta g_t')\|^2 \nonumber \\
    =& ~ \|x_t-\eta \nabla F(x_t)-(y_t-\eta \nabla F(y_t))-\eta (g_t- \nabla F(x_t)-g_t'+\nabla F(y_t)   )\|^2 \nonumber\\
    =& ~ \|x_t-\eta\nabla F(x_t)-(y_t-\eta \nabla F(y_t))\|^2 \nonumber\\ 
    &~-2\eta \langle  x_t-\eta\nabla F(x_t)-(y_t-\eta \nabla F(y_t)),g_t- \nabla F(x_t)-g_t'+\nabla F(y_t)\rangle \nonumber\\
    &~+ \eta^2 \|g_t-\nabla F(x_t)-g_t'+\nabla F(y_t)\|^2 \nonumber \\
    \stackrel{(i)}{\le} & ~ \|x_t-y_t\|^2-2\eta \langle  x_t-\eta\nabla F(x_t)-(y_t-\eta \nabla F(y_t)),g_t- \nabla F(x_t)-g_t'+\nabla F(y_t)\rangle \nonumber \\
    &~+ 4\eta^2 L^2, 
\end{align}
where $(i)$ follows from the non-expansive property of gradient descent on smooth convex function for $\eta \le 1/\beta$~\cite{hardt16}. 

Define $a_t:=-2\eta \langle  x_t-\eta\nabla F(x_t)-(y_t-\eta \nabla F(y_t)),g_t- \nabla F(x_t)-g_t'+\nabla F(y_t)\rangle$.
By Inequality~\eqref{eq: thingy} and the inductive hypothesis, we obtain
\begin{align*}
    \|x_{t+1}-y_{t+1}\|^2\le 4\eta^2L^2t+\sum_{i=1}^{t}a_t.
\end{align*}
It remains to bound $\sum_{i=1}^t a_i$.
Note that $\E[a_i \mid a_1,\cdots,a_{i-1}]=0$, and by Lemma~\ref{lm:inner_product_nSG} we know there is a constant $c > 0$ such that $a_i$ is $\nSG(c \eta L \|x_i-y_i\|)$ for all $i$.
Hence by Theorem~\ref{thm:hoeffding_nSG}, we know
\begin{align*}
    \mathbb{P}\left[\left|\sum_{i=1}^{t}a_i\right|\ge c\eta L\sqrt{\log(dT/\gamma)\sum_{i\le t}\|x_i-y_i\|^2}\right]\le 1-\zeta/T.
\end{align*}
Conditional on the event that $\|x_{i}-y_{i}\|\le c\sqrt{\log(dT/\zeta)}\eta L\sqrt{i}$ for all $i \leq t$ (which happens with probability $1-t\zeta/T$ by the  inductive hypothesis), we know
\begin{align*}
    \mathbb{P}\left[\left|\sum_{i=1}^{t}a_i\right|\ge c^2(t+1)L^2\eta^2\log(dT/\zeta) \middle| \|x_i-y_i\|\le c\log(dT/\zeta)\eta L\sqrt{i},\forall i\le t\right]\le 1-\zeta/T.
\end{align*}
Hence we know
\begin{align*}
    \mathbb{P}\left[\|x_{t+1}-y_{t+1}\|^2\ge c^2\log(dT/\zeta)\eta^2L^2(t+1) \middle|  \|x_i-y_i\|\le c\log(dT/\zeta)\eta L\sqrt{i},\forall i\le t \right]\le 1-\zeta/T.
\end{align*}
Combining the above pieces completes the inductive step, showing that $\|x_{t+1}-y_{t+1}\|\le c\sqrt{(t+1)\log(dT/\zeta)}\eta L + 2\eta L$ with probability at least $1-(t+1)\zeta/T$.
This completes the proof.

\end{proof}

By combining Lemmas~\ref{lemma 2} and~\ref{lem: FKT Lem4.3} with the localization proof technique of~\cite{fkt20}, we can now prove the excess risk guarantee of~\cref{thm: phased erm without reg - appendix}:
\begin{proof}[Proof of~\cref{thm: phased erm without reg} - Excess risk]

\textbf{Excess Risk:} First, we will argue that $\tilde{x}_i = \frac{1}{C} \sum_{j=1}^C \tilde{x}_{i,j}$ for all $i$ with high probability $\geq 1 - 3/nm$. Lemma~\ref{lemma 2} implies that \[
\|\tilde{x}_{i,j} - \tilde{x}_{i,j'}\| \leq \tau_i
\]
for all $i \in [l]$, $j, j' \in [C]$ with probability at least $1 - 1/nm$. Thus, $s_i(\tau_i) = C$ with probability at least $1 - 1/nm$. Now, conditional on $s_i(\tau_i) = C$, we have $\widehat{s}_i(\tau_i) \geq 4C/5$ for all $i$ with probability at least $1 - 1/nm$ by Laplace concentration and a union bound. Moroever, if $\|\tilde{x}_{i,j} - \tilde{x}_{i,j'}\| \leq \tau_i$ for all $j,j'$, then $p_{i,j} = 1$ for all $j$ and hence there are no outliers: $\calS_i = \{\tilde{x}_{i,j}\}_{j \in [C]}$. By a union bound, we conclude that $\calS_i = \{\tilde{x}_{i,j}\}_{j \in [C]}$ and hence $\tilde{x}_i = \frac{1}{\calS_i} \sum_{D_{i,j} \in \calS_i} \tilde{x}_{i,j}$ for all $i$ with probability at least $\geq 1 - 3/nm$. By the law of total expectation and Lipschitz continuity, it suffices to condition on this high probability good event that $\tilde{x}_i = \frac{1}{C} \sum_{j=1}^C \tilde{x}_{i,j}$ for all $i$: the total expected excess risk can only be larger than the conditional excess risk by an additive factor of at most $3LR/nm$. 

Now, conditional on $\tilde{x}_i = \frac{1}{\calS_i} \sum_{D_{i,j} \in \calS_i} \tilde{x}_{i,j}$ , Lemma~\ref{lem: FKT Lem4.3} and Jensen's inequality implies 
\begin{equation}
    \expec[F(\tilde{x}_i) - F(\tilde{x}_{i-1})] \lesssim \frac{\expec\|\tilde{x}_{i-1} - x_{i-1}\|^2}{\eta_i T_i} + \eta_i L^2  = \frac{d \sigma_{i-1}^2}{\eta_i T_i} + \eta_i L^2. 
\end{equation}

Next, let $x_0^* := x^* = \text{argmin}_{x \in \XX} F(x)$, and write
\begin{align*}
    \expec[F(x_l) - F^*] &= \sum_{i=1}^l \expec[F(x_i^*) - F(x_{i-1}^*)] + \expec[F(x_l) - F(x_l^*)] \\
    &\lesssim \frac{R^2}{\eta T_1} + \eta L^2 + 
    \sum_{i=2}^l \left[\eta_{i-1} L^2 d + \eta_i L^2 \right] + L^2 \sqrt{d} \eta_l \sqrt{T_l} \\
    &\lesssim \frac{R^2}{\eta T_1} + d \eta L^2 + L^2 \sqrt{d} \eta_l \sqrt{T_l}. 
\end{align*}
Plugging in the prescribed algorithmic parameters completes the excess risk proof. 
\end{proof}

\section{Proofs of Results in Section~\ref{sec: accelerated phased ERM}}
\label{app: accelerated phased ERM}

\begin{theorem}[Re-statement of~\cref{thm: privacy of accelerated phased ERM}]
\label{thm: privacy of acc phased appendix}
    Let $\eps \leq 10$, $q > 0$ such that $n^{1-q} > \frac{100 \log(20nmde^{\eps}/\delta)}{\eps (1 - (1/2)^q)}$. 
    Then, Algorithm~\ref{alg: accelerated phased ERM} is $(\eps, \delta)$-DP.
\end{theorem}

We require the following lemma, which is a direct consequence of~\cite[Lemma 3.5]{asi2023user}:

\begin{lemma} 
\label{lem: 3.5}
Consider~Algorithm~\ref{alg: accelerated minibatch}. Let $\DD'_i$ and $\DD_i'$ be two data sets that differ in the data of one user. Let $E_i = \{|v_i^t| \leq K_i/20 ~\forall t \in [T_i] \cap |\xi_i| \leq K_i/20 \}$. Define $E_i'$ similarly for independent draws of random Laplace noise: $E_i' = \left\{|(v_i^{t})'| \leq K_i/20 ~\forall t \in [T_i] \cap |\xi_i'| \leq K_i/20 \right\}$. Let $a_i^t = \ind(\widehat{s}_i^t(\DD_i) \geq 4K_i/5)$ and $b_i^t = \ind(\widehat{s}_i^t(\DD_i') \geq 4K_i/5)$ denote the concentration scores in iteration $t$. Then, conditional on $E_i \bigcap E_i'$ and conditional on $a_i^t = b_i^t$, there exists a coupling $\Gamma_i^t$ over $g_i^t(\DD_i)$ and $g_i^t(\DD_i')$ such that for $(h, h')$ drawn from $\Gamma_i$, we have \[
\|h - h'\| \lesssim \frac{\tau \log(1/\zeta)}{K_i}
\]
with probability at least $1 - \zeta$. 
\end{lemma}

\begin{proof}[Proof of~\cref{thm: privacy of acc phased appendix}]
Note that our assumption on $n^{1-q}$ being sufficiently large implies that $n_i \gtrsim \frac{\log(nmd/\delta)}{\eps}$ for all $i \in [l]$. By parallel composition~\cite{mcsherry2009privacy} and post-processing, it suffices to show that $\{\widehat{g}_i^t\}_{t=1}^{T_i}$ satisfies $(\eps, \delta)$-user-level DP for any $i \in [l]$. To that end, fix any $i \in [l]$ and let $\DD$ and $\DD'$ be adjacent datasets that differ in the data of one user such that $\DD_i \neq \DD'_i$. We will show that $\{\widehat{g}_i^t(\DD)\}_{t=1}^{T_i}$ and $\{\widehat{g}_i^t(\DD')\}_{t=1}^{T_i}$ are $(\eps, \delta)$-indistinguishable, which will imply that Algorithm~\ref{alg: accelerated phased ERM} is $(\eps, \delta)$-user-level DP. 
Let $E_i = \{|v_i^t| \leq K_i/20 ~\forall t \in [T_i] \cap |\xi_i| \leq K_i/20 \}$. Define $E_i'$ similarly for independent draws of random Laplace noise: 
$E_i' = \left\{|(v_i^{t})'| \leq K_i/20 ~\forall t \in [T_i] \cap |\xi_i'| \leq K_i/20 \right\}$. 
Our choice of $K_i \geq \frac{500\log(nmd e^{\eps}/\delta)}{\eps}$  ensures that 
\[
\pr\left(E_i \bigcap E_i'\right) \geq 1 - \delta/(10e^{\eps}),
\]
by Laplace concentration and a union bound. 
Let $\zeta := \delta/(10T_ie^\eps)$. 

Let $a_i^t = \ind(\widehat{s}_i^t(\DD) \geq 4K_i/5)$ and $b_i^t = \ind(\widehat{s}_i^t(\DD') \geq 4K_i/5)$. Note that if $a_i^t = b_i^t = 0$, then $\widehat{g}_i^t(\DD) = 0 = \widehat{g}_i^t(\DD')$.

   Conditional on the good event that $a_i^t = b_i^t$ for all $t$ and conditional on $E_i \bigcap E_i'$, we can bound the $\ell_2$-sensitivity of $g_i^t$ with high probability, via Lemma~\ref{lem: 3.5} and a union bound: 
   \begin{equation}
    \label{eq: sens bound}
    \|g_i^t(\DD) - g_i^t(\DD')\| \lesssim \frac{\tau \log(1/\zeta)}{K_i} \lesssim \frac{\tau \log(nme^{\eps}/\delta)}{K_i}
   \end{equation}
 for all $t \in [T_i]$  with probability at least $1 - T_i \zeta = 1 - \delta/(10 e^\eps)$.

   Note that $\{\widehat{s}_i^t(\DD)\}_{t=1}^{T_i}$ and $\{\widehat{s}_i^t(\DD')\}$ are $\eps/2$-indistinguishable by the DP guarantees of AboveThreshold in Lemma~\ref{thm:Above_Threshold}, since the sensitivity of $s_i^t$ is upper bounded by $2$. Therefore, \begin{equation}
   \label{eq: lap2}
       \pr(\{a_i^t\}_{t=1}^{T_i} = v, E_i) \leq e^{\eps/2} \left[\pr(\{b_i^t\}_{t=1}^{T_i} = v, E_i')  + \zeta \right]
   \end{equation}
   for any $v \in \{0,1\}^{T_i}$. 

\vspace{.2cm}
Now, by the sensitivity bound~\eqref{eq: sens bound}, the privacy guarantee of the Gaussian mechanism (Lemma~\ref{lem: gauss mech}) and our choice of $\sigma_i$, the advanced composition theorem (Lemma~\ref{thm: advanced composition}), and privacy amplification by subsampling (Lemma~\ref{lem: amp by subsampling}), we have \begin{equation}
\label{eq: gauss mech bound}
    \pr(\{\hat{g}_i^t(\DD)\}_{t=1}^{T_i} \in \calO \mid E_i \{a_i^t\}_{t=1}^{T_i} = v) \le e^{\eps/2}  \pr(\{\hat{g}_i^t(\DD')\}_{t=1}^{T_i} \in \calO \mid E'_i, \{b_i^t\}_{t=1}^{T_i} = v) + (T_i + 1)\zeta,
\end{equation}
for any event $\calO \subset \XX$. Here we also used the fact that $K_i \geq \frac{n_i \eps}{\sqrt{T_i}}$. 

\vspace{.2cm}
For short-hand, write $\{a_i^t\}_{t=1}^{T_i} = 1$ if $a_i^t = 1$ for all $t \in [T_i]$ and $\{a_i^t\}_{t=1}^{T_i} = 0$ if $a_i^t = 0$ for some $t \in [T_i]$; similarly for $b_i^t$. Then since the algorithm halts and returns $\{\hat{g}_i^t(\DD)\}_{t=1}^{T_i} = 0$ if  $\{a_i^t\}_{t=1}^{T_i} = 0$, we know that \begin{equation}
\label{eq: halt2}
  \pr(\{\hat{g}_i^t(\DD)\}_{t=1}^{T_i} \in \calO \mid E_i , \{a_i^t\}_{t=1}^{T_i} = 0) = \pr(\{\hat{g}_i^t(\DD')\}_{t=1}^{T_i} \in \calO \mid E'_i, \{b_i^t\}_{t=1}^{T_i} = 0),
\end{equation}

for any event $\calO \subset \XX$. 

\vspace{.2cm}
Combining the above pieces, we have \begin{align*}
    \pr(\{\hat{g}_i^t(\DD)\}_{t=1}^{T_i} \in \calO) &= \pr(\{\hat{g}_i^t(\DD)\}_{t=1}^{T_i} \in \calO \mid E_i) \pr(E_i) + \pr(\{\hat{g}_i^t(\DD)\}_{t=1}^{T_i} \in \calO \mid E_i^c) \pr(E_i^c) \\
    &\leq  \pr(\{\hat{g}_i^t(\DD)\}_{t=1}^{T_i} \in \calO \mid E_i, \{a_i^t\}_{t=1}^{T_i} = 1) \pr(E_i, \{a_i^t\}_{t=1}^{T_i} = 1) \\
    &\;\;\; + \pr(\{\hat{g}_i^t(\DD)\}_{t=1}^{T_i} \in \calO \mid E_i, \{a_i^t\}_{t=1}^{T_i} = 0)\pr(E_i, \{a_i^t\}_{t=1}^{T_i} = 0) + 2T\zeta \\
    &\stackrel{(i)}{\le} e^{\eps/2}\pr(\{\hat{g}_i^t(\DD')\}_{t=1}^{T_i} \in \calO \mid E_i', \{b_i^t\}_{t=1}^{T_i} = 1) e^{\eps/4} \left[\pr(E_i', \{b_i^t\}_{t=1}^{T_i} = 1) + T\zeta \right] \\
    &\;\;\;\; + \pr(\{\hat{g}_i^t(\DD')\}_{t=1}^{T_i} \in \calO \mid E_i', \{b_i^t\}_{t=1}^{T_i} = 0) e^{\eps/2} \left[\pr(E_i', \{b_i^t\}_{t=1}^{T_i} = 0) + T\zeta \right] + T\zeta \\
    &\leq e^{\eps/2} \pr(\{\hat{g}_i^t(\DD')\}_{t=1}^{T_i} \in \calO, E_i') + 4T\zeta\left(2 e^{\eps/2} + 1 \right) \\
    &\leq e^{\eps} \pr(\{\hat{g}_i^t(\DD')\}_{t=1}^{T_i} \in \calO) + \delta,
\end{align*}
where $(i)$ follows from inequalities~\eqref{eq: lap2}, \eqref{eq: gauss mech bound}, and \eqref{eq: halt2}. Thus, $\{\hat{g}_i^t(\DD)\}_{t=1}^{T_i}$ is $(\eps, \delta)$-user-level-DP, which implies the result. 

\end{proof}

\begin{theorem}[Formal statement of~\cref{thm: utility of accelerated phased ERM}]
\label{thm: utility of accelerated phased ERM - appendix verseion}
Let 
$\eps \le 10$ and $\delta < 1/(mn)$. Then, choosing $\lambda = \frac{L}{R}\left(\frac{1}{\sqrt{nm}} + \frac{\sqrt{d}}{{\eps n \sqrt{m}}}\right)$ and $p \geq 3q + 2.5 + \log_n(\sqrt{m})$ in Algorithm~\ref{alg: accelerated phased ERM} yields optimal excess risk:
    \[
    \expec F(x_l) - F^* \leq LR \cdot \widetilde{O}\left(\frac{1}{\sqrt{mn}} + \frac{\sqrt{d \log(1/\delta)}}{\eps n \sqrt{m}} \right).
    \]
   The gradient complexity of this algorithm is upper bounded by \[
mn\left(1 + \eps\left(\frac{\beta R}{L}\right)^{1/4}\left((mn)^{1/8} \wedge \left(\frac{\eps^2 n^2 m}{d} \right)^{1/8} \right) \right) + \sqrt{\frac{\beta R}{L}}\left(\frac{n^{1/4} m^{5/4}}{\eps} + \left(\frac{n^{1/2}m^{5/4}}{d^{1/4} \eps^{1/2}} \right)\right).
    \]
\end{theorem}

We will need the following bound on the excess empirical risk of accelerated (noisy) SGD for the proof of~\cref{thm: utility of accelerated phased ERM - appendix verseion}: 
\begin{lemma}{\cite[Proposition 7]{ghadimilan2}}
\label{lem: ghadimi}
Let $x^T$ be computed by $T$ steps of (multi-stage) Accelerated Minibatch SGD on $\lambda$-strongly convex, $\beta$-smooth $\hat{F}$ with unbiased stochastic gradient estimator $g^t$ such that $\expec \|g^t - \nabla \hf(x^t)\|^2 \leq V^2$ for all $t\in [T]$. Then, \[
\expec[\hf(x^T) - \min_{x \in \XX} \hf(x)] \lesssim [\hf(x^0) - \min_{x \in \XX} \hf(x)] \exp\left(-T\sqrt{\frac{\lambda}{\beta}}\right) + \frac{V^2}{\lambda T}.
\]
\end{lemma}

\begin{remark}
As noted in Lemma~\ref{lem: ghadimi}, we technically need to call a \textit{multi-stage implementation} of~Algorithm~\ref{alg: accelerated minibatch} in line 7 of~Algorithm~\ref{alg: accelerated phased ERM} (as in~\cite{ghadimilan2}) to get the desired excess risk bound for minimizing the regularized ERM problem in each iteration. For improved readability, we omitted these details in the main body. 
\end{remark}

Next, we obtain a bound on the variance of the noisy stochastic minibatch gradient estimator $\widehat{g}_i^t$ in~Algorithm~\ref{alg: accelerated minibatch}, which can then be plugged in for $V^2$ in Lemma~\ref{lem: ghadimi}. 
\begin{lemma}[Re-statement of Lemma~\ref{lem: variance bound}]
Let $\delta \leq 1/(nm), \eps \lesssim 1$. Denote $\widetilde{F}_i(x) := 
\frac{1}{n_i} \sum_{Z_{i,j} \in D_{i}} \hat{F}(x, Z_{i,j})
$.
Then, conditional on $\calS_i^t = D_i^t$ for all $i \in [l], t \in [T_i]$, we have \[
\expec\| g_i^t - \nabla \widetilde{F}_i(x_{i-1}^t) \|^2 \lesssim \frac{L^2 \log(ndm)}{Km}
\]
for all $i \in [l], t \in [T_i]$, where the expectation is over both  
the random i.i.d. draw of $\DD = (Z_1, \ldots, Z_n) \sim P^{nm}$ and the randomness in Algorithm~\ref{alg: accelerated phased ERM}. 
\end{lemma} 
\begin{proof}
By \cite[Lemma A.1]{lei17}, we know that, conditional on the draw of the data  $D_i$ and for fixed $x_{i-1}^t$, the variance of the minibatch estimator of the gradient of the empirical loss is
\begin{align}
\label{eq:A}
\expec \left[\left \|g_i^t - \nabla \widetilde{F}_i(x_{i-1}^t) \right\|^2
\Bigg | \normalsize D_i, x_{i-1}^t \right] &=
\expec_{\{i_l\}_{l=1}^K \sim \text{Unif}([n])} \left[\left \|\frac{1}{Km}\sum_{l=1}^K \sum_{j=1}^m \nabla f(x_{i-1}^t, z^t_{i_l, j})) - \nabla \widetilde{F}_i(x_{i-1}^t) \right\|^2
\Bigg | \normalsize D_i, x_{i-1}^t \right] \nonumber \\ 
&\leq \frac{\ind(K=n)}{K} \frac{1}{n} \sum_{i=1}^n \left\| \frac{1}{m} \sum_{j=1}^m [\nabla f(x_{i-1}^t, z^t_{i,j}) - \nabla \widetilde{F}_i(x_{i-1}^t) ] \right\|^2. 
\end{align}
Recall $\widetilde{F}_i(x) := \frac{1}{n_i m} \sum_{z \in D_i} f(x, z)$ is the empirical loss of $D_i$.

Now, for any fixed $x$ and any $Z \in D_i$, Hoeffding's inequality implies that \[
\|\nabla \hf(x, Z) - \nabla F(x)\| \leq \tau = O\left( \frac{L \sqrt{\log(nd/\gamma)}}{\sqrt{m}} \right)
\]
with probability at least $1 - \gamma$, where $\hf(x, Z) := \frac{1}{m} \sum_{z \in Z} f(x, z)$ is user $Z$'s empirical loss.
Thus, by the stability of user-level DP (see \cite[Theorem 3.4]{bassily2023user}), for any $i \in [l], t \in [t]$, we have that 
\begin{equation}
\label{eq: tau bound}
\| \nabla \hf(x_{i-1}^t, Z) - \nabla F(x_{i-1}^t) \| \leq \tau 
\end{equation}
for all $Z \in \DD$ with probability at least $1 - \gamma' = n(e^{2\eps} \gamma + \delta)$, since $x_{i-1}^t$ is $(\eps, \delta)$-user-level DP. 
To make $\gamma' \lesssim 1/m$, we choose $\gamma = 1/mn$ and use the assumptions that $\eps \lesssim 1$ and $\delta \leq 1/mn$. Thus, for any fixed $i, t$ we have \[
\| \nabla \hf(x_{i-1}^t, Z) - \nabla F(x_{i-1}^t) \| \lesssim \frac{L \sqrt{\log(n^2 m d)}}{\sqrt{m}}
\]
for all $Z \in \DD$ with probability at least $1 - 1/m$, which implies \[
\expec\| \nabla \hf(x_{i-1}^t, Z) - \nabla F(x_{i-1}^t) \|^2 \lesssim \frac{L^2 \log(n m d)}{m}.
\]

This also implies \[
\expec\| \nabla \hf_{\DD}(x_{i-1}^t) - \nabla F(x_{i-1}^t) \|^2 \lesssim \frac{L^2 \log(n m d)}{m},
\]
by Jensen's inequality, where $\hf_{\DD}(x)$ is the empirical loss over the entire data set $\DD$.

Plugging the above bounds into~\eqref{eq:A} then yields
\begin{align*}
\expec\| g_i^t - \nabla \widetilde{F}_i(x_{i-1}^t) \|^2 \lesssim \frac{L^2 \log(ndm)}{Km} &=
\expec_{\DD \sim P^{nm}, \{i_l\}_{l=1}^K \sim \text{Unif}([n])} \left[\left \|\frac{1}{Km}\sum_{l=1}^K \sum_{j=1}^m \nabla f(x_{i-1}^t, z_{i_l, j}^t) - \nabla \hat{F}_{\DD}(x_{i-1}^t) \right\|^2
\right] \\
&\lesssim \frac{\ind(K=n)}{K} \cdot \frac{L^2 \log^2(nmd)}{m},
\end{align*}
completing the proof.

\end{proof}

We are now ready to prove~\cref{thm: utility of accelerated phased ERM - appendix verseion}: 
\begin{proof}[Proof of~\cref{thm: utility of accelerated phased ERM - appendix verseion}]
\textbf{Excess risk:}
Note that the assumption in the theorem ensures that $K_i \leq n_i$ for all $i$. 
By similar arguments to those used in \cite[Proposition 3.7]{asi2023user}, we will show that with high probability $\geq 1 - 2/(nm)$, for all $i\in[l], t\in [T_i]$, $\mathcal{S}_i^t = D_i^t$ and hence $g_i^t$ is an unbiased estimator of $\nabla \hf_{D_i^t}(x_i^t)$. To show this, first note that for any $\gamma >0$ and any fixed $x$, \[
\|\nabla \hf(x, Z_j) - \nabla F(x)\| \leq \frac{L \log(nd/\gamma)}{\sqrt{m}}
\]
with probability at least $1 - \gamma/K_i$ by Hoeffding's inequality (see \cite[Lemma 4.3]{asi2023user}). Next, we invoke the stability of differential privacy to show that for all $t \in [T_i]$, $(q_t(Z_{i,1}^t), \ldots, q_t(Z_{i, K_i}^t))$ is \textit{$\tau$-concentrated} (i.e. there exists $q^*  \in \mathbb{R}^d$ such that $\|q_t(Z_{i,j}^t - q^*\| \leq \tau$) with probability at least $1 - T_i(e^{2\eps} \gamma + \delta)$ (see \cite[Theorem 4.3]{bassily2023user}). By a union bound and the choice of $\gamma = 1/[(nm)^{5/4} \log(ndm) e^{2\eps}]$, we have that $(q_t(Z_{i,1}^t), \ldots, q_t(Z_{i, K_i}^t))$ is $\tau$-concentrated for all $i \in [l], t \in [T_i]$ with probability at least $1 - 1/nm$. 
Now, $\tau$-concentration of $(q_t(Z_{i,1}^t), \ldots, q_t(Z_{i, K_i}^t))$ implies $s_i^t(\tau) = K_i$. Further, $s_i^t(\tau) = K_i$ implies $\hat{s}_i^t(\tau) \geq 4K_i/5$ with probability at least $1 - \zeta$ if $K_i \geq 500 \log(nm/\zeta)$, by Laplace concentration and a union bound.
Next, note that $\tau$-concentration of $(q_t(Z_{i,1}^t), \ldots, q_t(Z_{i, K_i}^t))$ implies $p_i^t(Z_{i,j}^t) = 1$ for all $j \in [K_i]$ and $\mathcal{S}_i^t = D_i^t$. Thus, $\mathcal{S}_i^t = D_i^t$ for all $i, t$ with probability at least $1 - 1/nm$. Setting $\zeta = 1/nm$ and using a union bound shows that with probability at least $1 - 2/nm$, we have $\mathcal{S}_i^t = D_i^t$ and $g_i^t = \nabla \hf_{D_i^t}(x_i^t)$ for all $i,t$.

Now, Lemma~\ref{lem: ghadimi} implies that, \textit{if the outlier-removal procedure in~Algorithm~\ref{alg: accelerated minibatch} leads to an unbiased gradient estimator $g_i^t$} (line 12) for all $t \in [T_i = \widetilde{\Theta}(1 + \sqrt{\beta/\lambda_i})]$, then 
\begin{equation}
\label{eq: accel bd}
\expec[\hf_i(x_i^{T_i}) - \min_{x \in \XX} \hf_i(x)] \lesssim \frac{V_i^2}{\lambda_i T_i},
\end{equation}
where $V_i^2 = \max_{t \in [t_i]} \expec\| \widehat{g}_i^t - \nabla \hf_i(x_i^t)\|^2 \lesssim d \sigma_i^2 + \frac{\log(ndm) L^2}{K_i m}$ (unconditionally, after taking expectation over the random draw of $\DD \sim P^{nm}$, by Lemma~\ref{lem: variance bound}). We have shown that the event $\text{GOOD} := \{g_i^t = \nabla \hf_{D_i^t}(x_i^t)$ for all $i \in [l],t \in [T_i]\}$ occurs with probability at least $1 - 2/nm$. We will condition on GOOD for the rest of the proof: note that the Lipschitz assumption implies that the total (unconditional) excess risk will only by larger than the conditional (on GOOD) excess risk by an additive factor of at most $2LR/\sqrt{nm}$. 

By stability of regularized ERM (see \cite{shalev2009stochastic}), we have
\begin{equation}
\label{eq: stability of reg ERM bd}
\expec[F(x_i^*) - F(y)] \lesssim \frac{L^2}{\lambda_i n_i m} + \lambda_i \expec[\|x_{i-1} - y\|^2] 
\end{equation}
for all $i$, where $x_i^* := \text{argmin}_{x} \hf_i(x)$. By strong convexity and~\eqref{eq: accel bd}, we have 
\begin{equation}
    (\lambda_i/2) \expec\|x_i - x_i^*\|^2 \leq \expec \hf_i(x_i) - \hf_i^* \lesssim \frac{d\sigma_i^2}{\lambda_i T_i} + \frac{L^2 \log(ndm)}{\lambda_i T_i K_i m}. 
\end{equation}

Thus, \begin{equation}
\label{eq: c}
    \expec\|x_i - x_i^*\|^2 \lesssim \frac{d \sigma_i^2}{\lambda_i T_i} + \frac{L^2 \log(ndm)}{\lambda_i T_i K_i m} \lesssim \frac{d \tau^2 \log(1/\delta)}{\lambda_i^2 \eps^2 n_i^2} + \frac{L^2 \log(ndm)}{\lambda_i^2 T_i K_i m}
\end{equation}

Now, letting $x_0^* := x^* = \text{argmin}_{x \in \XX} F(x)$ and hiding logarithmic factors, we have: 
\begin{align*}
    \expec[F(x_l) - F^*] &= \sum_{i=1}^l \expec[F(x_i^*) - F(x_{i-1}^*)] + \expec[F(x_l) - F(x_l^*)] \\
    &\lesssim \frac{L^2}{\lambda_1 n_1 m} + \lambda_1 R^2 + \sum_{i=2}^l \expec\left[\frac{L^2}{\lambda_i n_i m} + \lambda_i\|x_{i-1} - x^*_{i-1}\|^2\right] + L \expec\|x_l - x_{l}^*\| \\
    &\lesssim \frac{L^2}{\lambda n m} + \lambda R^2 + \sum_{i=2}^l \left[\frac{L^2}{\lambda_i n_i m} + \lambda_i \left(\frac{d \tau^2 \log(1/\delta)}{\lambda_{i-1}^2 \eps^2 n_{i-1}^2} + \frac{L^2}{\lambda_{i-1}^2 T_{i-1} K_{i-1} m}\right)\right] \\
    &\;\;\; + L \frac{\sqrt{d} \tau \sqrt{T_l \log(1/\delta)}}{\lambda_l \eps n_l},
\end{align*}
where the first inequality used \eqref{eq: stability of reg ERM bd} and Lipschitz continuity, the second inequality used \eqref{eq: c}.

Note that $K_i T_i \geq n_i$. Further, our choice of sufficiently large $p$ makes $\lambda_l$ large enough that $ L \frac{\sqrt{d} \tau \sqrt{T_l \log(1/\delta)}}{\lambda_l \eps n_l} \leq \frac{LR \sqrt{d}}{\eps n \sqrt{m}}$. Therefore, upper bounding the sum by it's corresponding geometric series gives us 
\begin{align}
\expec[F(x_l) - F^*] &\lesssim \frac{LR \sqrt{d}}{\eps n \sqrt{m}} + \frac{L^2}{\lambda}\left(\frac{1}{n m} + \frac{d \tau^2 \log(1/\delta)}{\eps^2 n^2}\right) + \lambda R^2. 
\end{align}
Plugging in $\lambda$ completes the excess risk proof. 

\textbf{Gradient Complexity:} The gradient complexity is $\sum_{i=1}^l T_i K_i m$. Plugging in the prescribed choices of $T_i$ and $K_i$ completes the proof. 
\end{proof}

\section{Details on the non-smooth algorithm and the proof of Theorem~\ref{thm: nonsmooth}}
\label{app: nonsmooth}

For any loss function $f(\cdot,z)$, we define the convolution function $f_r(\cdot,z):=f(\cdot,z)*n_r$ where $n_r$ is the uniform density in the $\ell_2$ ball of radius $r$ centered at the origin in $\R^d$.
Specifically, $n_r(y)= \frac{\Gamma(\frac{d}{2}+1)}{\pi^{\frac{d}{2}}r^d}$ for $\|y\|\le r$, and $n_r(y)=0$ otherwise.
For simplicity, we omit the dependence on $z$ in the following Lemma:

\begin{lemma}[Randomized Smoothing, \cite{yns11,dbw12}]
\label{lm:random_smooth}
For any $r>0$,
let $\XX_r:=\XX+\{x\in\R^d:\|x\|\le r\}$.
If $f$ is convex and $L$-Lipschitz over $\XX_r$,
then the convolution function $f_r$ has the following properties:
\begin{itemize}
    \item $f_r(x)\le f(x)\le f_r(x)+Lr$, for all $x\in \XX$.
    \item $f_r$ is $L$-Lipschitz and convex.
    \item $f_r$ is $\frac{L\sqrt{d}}{r}$-smooth.
    \item For random variables $y\sim n_r$, we have 
     $   \E_{y}[\nabla f(x+y)]=\nabla f_r(x).$
\end{itemize}
\end{lemma}

The following lemma can be easily seen from the proofs of~\cref{thm: privacy of accelerated phased ERM,thm: utility of accelerated phased ERM}:
\begin{lemma}[Privacy and utility of~Algorithm~\ref{alg: accelerated phased ERM} for general $K_i, T_i$]
\label{meta lemma}
Let $\eps \leq 10$, $q > 0$ such that $n^{1-q} > \frac{100 \log(20nmde^{\eps}/\delta)}{\eps (1 - (1/2)^q)}$. 
\begin{itemize}
    \item If $K_i \gtrsim \frac{n_i \eps}{\sqrt{T_i}} + \frac{\log(nmde^{\eps}/\delta)}{\eps}$, then~Algorithm~\ref{alg: accelerated phased ERM} is $(\eps, \delta)$-user-level DP. 
    \item If $T_i K_i \geq n_i$ and $T_i \gtrsim (1 + \sqrt{\beta/\lambda_i}) \log(ndm)$ for all $i$, then~Algorithm~\ref{alg: accelerated phased ERM} achieves optimal excess risk. 
\end{itemize}
\end{lemma}

\begin{theorem}[Formal statement of~\cref{thm: nonsmooth}]
Let $\eps \le 10$, $\delta < 1/(mn)$, and $q > 0$ such that $n^{1-q} > \frac{100 \log(20nmde^{\eps}/\delta)}{\eps (1 - (1/2)^q)}$. 
Suppose that for any $z$, $f(,z)$ is convex and $L$-Lipschitz over $\XX_r$ for $\XX_r:=\XX+\{x\in\R^d:\|x\|\le r\}$ where $r=\frac{\sqrt{d}}{\eps n\sqrt{m}}R$.
Then, 
    running Algorithm~\ref{alg: accelerated phased ERM} with functions $\{f_r(x;z)\}_{z \in \DD}$ yields optimal excess risk:
    \[
    \expec F(x_l) - F^* \leq LR \cdot \widetilde{O}\left(\frac{1}{\sqrt{mn}} + \frac{\sqrt{d \log(1/\delta)}}{\eps n \sqrt{m}} \right).
    \]
    The gradient complexity of this algorithm is upper bounded by \[
   mn\left(1 +n^{3/8}m^{1/4}\eps^{1/4}\right).
    \]
\end{theorem}

\begin{proof}
    By Lemma~\ref{lm:random_smooth} and our choice of $r$, we have $|f_r(x,z)- f(x,z)|\le Lr=O(LR\frac{\sqrt{d}}{\eps n\sqrt{m}})$.
    Set $\lambda=\frac{1}{\sqrt{mn}}$.
    Then we know that 
    \begin{align*}
        \E F(x_l)-F^* \le & \E \left[F_r(x_l)-F_r^*\right] +O(LR\frac{\sqrt{d}}{\eps n\sqrt{m}}).    \end{align*}

    Further, $F_r$ is $\beta$-smooth for $\beta\le\frac{L}{R}\eps n\sqrt{m}$.
   Set $T_i = (1 + \sqrt{\beta/\lambda_i}) \log(ndm) = 1 + n_i^{3/4}m^{1/2}\eps^{1/2}\log(ndm)$ and $K_i = \frac{n_i\eps}{\sqrt{T_i}} + \frac{\log(nmde^\eps/\delta)}{\eps}$. Then
    Lemma~\ref{meta lemma} implies that~Algorithm~\ref{alg: accelerated phased ERM} is $(\eps,\delta)$ user-level DP, and yields the excess risk bound
    \begin{align*}
       \E F_r(x_l)-F_r^*\le LR \cdot \Tilde{O}\left( \frac{1}{\sqrt{mn}}+\frac{\sqrt{d\log(1/\delta)}}{\eps n\sqrt{m}}\right),
    \end{align*}
    as desired.
    The number of gradient evaluations is
    \begin{align*}
    \sum_{i=1}^l T_i K_i m \lesssim mn\left(1 +n^{3/8}m^{1/4}\eps^{1/4}\right).
    \end{align*}
    This completes the proof.
    \end{proof}

\section{Limitations}
\label{app: limitations}
Our work weakens the assumptions on the smoothness parameter and the number of users that are needed for user-level DP SCO. Nevertheless, our results still require certain assumptions that may not always hold in practice. For example, we assume convexity of the loss function. In deep learning scenarios, this assumption does not hold and our algorithms should not be used. Thus, user-level DP \textit{non-convex} optimization is an important direction for future research~\cite{lowymake}. Furthermore, the assumption that the loss function is convex and uniformly Lipschitz continuous may not hold in certain applications, motivating the future study of user-level DP stochastic optimization with heavy tails~\cite{lowy2023largelip,asi2024private}.

Our algorithms are also faster than the previous state-of-the-art, including a linear-time Algorithm~\ref{alg: phased ERM without regularization} with state-of-the-art excess risk. However, our error-optimal accelerated Algorithm~\ref{alg: accelerated phased ERM} runs in super-linear time. Thus, in certain applications where a linear-time algorithm is needed due to strict computational constraints, Algorithm~\ref{alg: phased ERM without regularization} should be used instead. 

\section{Broader Impacts}
\label{app: broader impacts}

Our work on differentially private optimization for machine learning advances the field of privacy-preserving ML by developing techniques that protect the privacy of individuals (users) who contribute data. The significance of privacy cannot be overstated, as it is a fundamental right enshrined in various legal systems worldwide. However, the implications of our work extend beyond its intended benefits, and it is essential to consider both potential positive and negative impacts.

\paragraph{Positive Impacts:}
\begin{enumerate}
    \item Enhanced Privacy Protections: By incorporating differential privacy into machine learning models, we can provide strong privacy guarantees for individuals, mitigating the risk of personal data being exposed or misused.
    \item Ethical Data Utilization: DP ML enables organizations to leverage data while adhering to ethical standards and privacy regulations, fostering trust among users and stakeholders.
    \item Broad Applications: The techniques we develop can be applied across diverse domains, including healthcare, finance, and social sciences, where sensitive data is prevalent. This broad applicability can drive innovations while maintaining privacy.
    \item Educational Advancement: Our research contributes to the growing body of knowledge in privacy-preserving technologies, serving as a valuable resource for future studies and fostering an environment of continuous improvement in privacy practices.
\end{enumerate}

\paragraph{Potential Negative Impacts:}
\begin{enumerate}
    \item Misuse by Corporations and Governments: There is a risk that our algorithms could be exploited by entities to justify the unauthorized collection of personal data under the guise of privacy compliance. Vigilant oversight and clear regulatory frameworks are necessary to prevent such abuses.
    \item Decreased Model Accuracy: While DP ML provides privacy benefits, it can also lead to reduced model accuracy compared to non-private models. This trade-off might have adverse consequences, such as less accurate medical diagnoses or flawed economic forecasts. For example, an overly optimistic prediction of environmental impacts due to lower accuracy could be misused to weaken environmental protections.
\end{enumerate}

While recognizing the potential for misuse and the trade-offs involved, we firmly believe that the advancement and dissemination of differentially private machine learning algorithms offer a net benefit to society. By addressing privacy concerns head-on and advocating for responsible use, we aim to contribute positively to the field of machine learning and uphold the fundamental right to privacy. Through ongoing research, collaboration, and education, we strive to enhance both the capabilities and ethical foundations of machine learning technologies.

\end{document}